\documentclass{article}
\usepackage{ijcai26}

\usepackage{times}
\usepackage{soul}
\usepackage{url}
\usepackage[hidelinks]{hyperref}
\usepackage[utf8]{inputenc}
\usepackage[small]{caption}
\usepackage{graphicx}
\usepackage{amsmath}
\usepackage{amsthm}
\usepackage{booktabs}
\usepackage{algorithm}
\usepackage{algorithmic}
\usepackage[switch]{lineno}

\usepackage{xspace}
\usepackage{comment}
\usepackage{array}
\usepackage{framed}
\usepackage{physics}
\usepackage{amsfonts}

\usepackage{multirow}
\usepackage{makecell}
\usepackage{rotating}

\usepackage{appendix}

\usepackage{tikz}
\usepackage{caption}
\usepackage{subcaption}
\usetikzlibrary{shapes, arrows, positioning, fit, backgrounds, calc, decorations.pathreplacing}

\usepackage[english]{babel}
\usepackage[autostyle, english = american]{csquotes}
\MakeOuterQuote{"}

\newcommand{\framework}{\textsc{QLens}\xspace} 
\newcommand{\RomanNumeralCaps}[1]{\MakeUppercase{\romannumeral #1}} 
\newtheorem{theorem}{Theorem}

\title{\framework: Towards A Quantum Perspective of Transformers}

\author{
Aditya Gupta$^1$
\and
Kirandeep Kaur$^2$\and
Vinayak Gupta$^2$\\
Chirag Shah$^2$\\
\affiliations
$^1$Issaquah High School, Washington, USA\\
$^2$University of Washington, Washington, USA\\
\emails
aditya441g@gmail.com,
kaur13@uw.edu,
guptavinayak51@gmail.com,
chirags@uw.edu
}

\begin{document}

\maketitle

\begin{abstract}
    Current methods for understanding Transformers are successful at extracting intermediate output probability distributions during inference. However, these approaches function as limited diagnostic checkpoints, lacking a mathematical framework for modeling how each layer facilitates transitions between these distributions. Inspired by this gap, we turn to quantum mechanics, a field possessing a pre-built mathematical toolkit for describing the evolution of probability distributions from its study of stochastic particle measurements. We propose \framework a novel attempt to develop a physics-based perspective on the Transformer generation process. Under \framework, these neural networks are studied by converting their latent activations into a state vector in a Hilbert space derived from the model's output units. The evolution of this state through subsequent hidden layers is modeled with unitary operators and analogously defined Hamiltonians. To demonstrate \framework's potential, we conduct a proof-of-concept by probing three one-layer Transformers on common deployment tasks to investigate the influence of individual layers in model prediction trajectories. We present our work as a foundation for interdisciplinary methods to be leveraged towards a broader understanding of Transformers. Our code is available on \href{https://github.com/g2116201/QLens}{GitHub}.
\end{abstract}

\section{Introduction}
Probabilistic distributions lie at the heart of Transformers~\cite{Vaswani2017AttentionIA}. These attention-based models iteratively transform their residual stream, gradually converging on a hidden state which can be mapped to prediction probabilities. The Transformer architecture has proven to be remarkably effective at generalizing across statistical variations in real-world data~\cite{Brown2020LanguageMA}, accommodating them by learning distinct layer-by-layer refinement trajectories. Intriguingly, the process of a probabilistic system undergoing sequential evolution is not restricted to these deep models but is also a fundamental component of quantum mechanics (QM). This branch of physics reveals that particle measurements follow a distribution derived from a temporarily evolving wavefunction, where each observation's result is obtained at random, consistent with the final distribution. This conceptual parallel suggests that the mathematical language of QM could be leveraged to provide a novel interpretation of Transformers by formalizing each layer's contribution in the prediction path.

Current approaches to interpretability have established tools for iterative inference~\cite{jastrzebski2018residual} in Transformers, which are the focus of this work. Prominently, the Logit Lens~\cite{nostalgebraist2020logitlens} has emerged as a prominent model that uses the final layer of a Transformer to extract its intermediate predictions during inference. Although the original lens was found to be a biased estimator of final outputs, Belrose et al. improved its design by incorporating a tunable constant term to account for the average residual update learned by the model~\cite{Belrose2023ElicitingLP}. Their Tuned Lens method has facilitated the analysis of example difficulty~\cite{Belrose2023ElicitingLP}, epistemic uncertainty in prediction trajectories~\cite{Kim2025OnTE}, and stages of inference~\cite{LAD2024StagesOfInference} in Transformers. However, while the Tuned Lens is a functional diagnostic tool for identifying the various points that comprise a model's prediction trajectory, it lacks a mathematical model for studying how individual layers facilitate transitions along this path.

To make inroads towards overcoming this difficulty, we look to physics for inspiration, drawing from its extensive theoretical backbones for observable phenomena. Whereas previous efforts to unite physics with interpretability studies have been largely reserved to its classical branches~\cite{Liu2025NeuralTL,BETTI2016}, we propose a novel QM-inspired perspective. Our \underline{Q}uantum Lens (\framework) harnesses quantum concepts to provide a  mathematical framework to model layer transitions. By formalizing the prediction path as a sequence of unitary state evolutions, \framework moves beyond the capabilities of the Tuned Lens to characterize each layer's contribution to the final output probability distribution. We then apply \framework to three simple Transformers, illustrating that it can mathematically describe the effect of each layer in the output space and derive insights from the resulting data.

To ground the methodology of \framework, we begin by deriving the closest equivalents of some of quantum mechanical postulates in the context of Transformers. We illustrate that the output space of a Transformer can be converted to an analogous orthonormal Hilbert basis for the model, with each output unit corresponding to a distinct state that a Transformer can produce. Exercising this framing, \framework views a Transformer's latent representation as specifying a \textit{state vector}, using which the probability distribution for subsequent tokens can be tracked during inference. Transformer layers are formulated as \textit{unitary operators} evolving their input state vectors according to the Schrödinger equation. This framework enables the definition of a corresponding \textit{Hamiltonian} for each layer, which provides a dual perspective of layers that aligns with residual addition.

With these analogies, \framework institutes a theoretical basis to translate the QM formulations to Transformers. Our work both elicits conceptual parallels between these two disciplines and identifies the open areas in which the ties between the two fields may be strengthened. Through this fresh approach, we hope to open a promising pathway towards a more profound understanding of Transformers. Our key contributions are summarized below:
\begin{enumerate}
    \item We demonstrate the thematic resemblance between QM and Transformers, uncovering the nearest equivalents of the postulates of QM in this domain.
    \item Through these definitions, we propose \framework, an end-to-end quantum mechanical analogy of the Transformer inference process.
    \item We demonstrate proof of concept by applying \framework to interpret the layers of three Transformer-based models covering multiple task domains.
    \item We critically analyze the current strengths and limitations of \framework to promote further exploration of this novel perspective of Transformer.
\end{enumerate}

\section{Related Work}
Interpretability is the procedure of converting one explanation of a concept into another, more human-understandable form~\cite{Zhao2023ExplainabilityFL,Erasmus2020WhatII}. As neural network architectures increasingly scale~\cite{Radford2019LanguageMA,Brown2020LanguageMA,Dubey2024TheL3}, it has become necessary to establish a conceptual grasp of these models to ensure model faithfulness beyond performance metrics~\cite{Carvalho2019MachineLI}. Present approaches to interpretability include classifier-based probes~\cite{jawahar-etal-2019-bert}, activation explanation~\cite{Dalvi2018WhatIO,Rai2024AnIO}, and mechanistic interpretability with circuits~\cite{Bricken2023Circuits,NEURIPS2023_34e1dbe9}. Jawahar et al. (2019) train a linear probe to categorize BERT~\cite{Devlin2019BERTPO} outputs by linguistic features of interest. Their findings reveal that layers at different positions in BERT's architecture specialize to various levels of linguistic characteristics: earlier layers identify surface-level syntactic patterns while later ones surface semantic concepts. Dalvi et al. (2018) adopt a more fine-grained methodology, training classifiers to determine correlations between neuron activations and the presence of lexical structures, illustrating that individual directions in the model's latent space can represent specific properties. Bricken et al. (2023) observe the challenge of polysemanticity---where individual neurons respond to multiple distinct combinations of inputs---and instead use sparse autoencoders to display that groups of neurons and their connections form multi-layer "circuits" to capture semantic features. However, these low-level methods localize to a limited range of behaviors and do not explain information flow through the model as a whole. In contrast, we argue that procedures based on physical analogies can more holistically achieve the goal of interpretability by adopting an overarching theoretical viewpoint to reframe the complex workings of Transformers in the context of well-studied natural phenomena.

Integrating physics into machine learning is an idea that returns to the roots of the latter field. Prominently, Hopfield networks~\cite{Hopfield1982} use a Hebbian learning algorithm to train a collection of neurons to store a desired memory in the form of bits. The state of the network at any time is associated with a particular energy function that derives from the Ising models of statistical mechanics. These networks themselves are only a subset of the generalized group of Boltzmann machines~\cite{Ackley1985}. Although such physics-inspired model frameworks are valuable for their inherently interpretable design, their specific architecture requirements and symmetric weight matrices limit the applicability of their insights to the current state-of-the-art Transformers despite contemporary breakthroughs~\cite{Ramsauer2020HopfieldNI}.

Recently, physics has been translated to probe the depths of more modern models and garner insights regarding their internal processes. Large language model (LLM) training dynamics have been explored from the perspectives of both thermodynamics~\cite{Liu2025NeuralTL} and classical mechanics~\cite{BETTI2016}, demonstrating that laws paralleling those that govern the universe may be fundamental to the learning process. In comparison to classical physics, however, the applications of quantum mechanics towards developing a comprehensive understanding of Transformer functions have been sparsely explored. Current interdisciplinary works between these two domains have primarily fallen under the quantum machine learning (QML)~\cite{Jiao2024AIPhysics} or machine learning for quantum mechanics (ML4QM)~\cite{Dawid2022ModernAO} paradigm. QML aims to harness potential computational advantages wrought by quantum computers for training~\cite{Kong2025QuantumEnhancedLE} and implementing~\cite{Cherrat2024quantumvision,CHEN2025MixedStateAttention,Li2022QuantumSN,Smaldone2025HybridTransformer} Transformers efficiently. For instance, Smaldone et al. (2025) developed a hybrid quantum-classical approach for self-attention that reduced its theoretical time complexity dependence on the embedding dimension from linear to logarithmic~\cite{Smaldone2025HybridTransformer}. In contrast to QML, ML4QM uses Transformers as a problem-solving tool in QM, applying it to simulate the complex dynamics of multi-particle systems~\cite{Zhang2022TransformerQS,Lange2024TransformerNN}. Contrary to these approaches, we propose using QM as a mathematical framework to sequentially examine Transformer layers and explain the process towards Transformers converge on their outputs. In doing so, we uncover novel analogies for the forward pass of the generation process and ground future work in this nascent area of quantum-inspired interpretability.


\section{Background}
\framework lies at an intersection of Transformers, QM, and the Tuned Lens. To lay the groundwork for it, we outline the central tenets of the latter two domains in this section. The founding postulates of quantum mechanics provide the language through which \framework is developed, whereas the Tuned Lens helps translate between QM's probabilities and Transformer hidden states. Through these concepts, \framework synthesizes a novel mathematical perspective towards conceptualizing Transformers.

\subsection{Quantum Mechanics}
Emerging in the twentieth century~\cite{blackbody_radiation,bohr_model}, QM reshaped the deterministic view held by most prominent physicists of the time. Led by figures such as Max Planck, Niels Bohr, Erwin Schrödinger, and Werner Heisenberg, the quantum revolution established a new probabilistic formulation of physics. We review the basics of this theory to prepare for its implementation in \framework.

Rather than possessing a definite state, quantum particles exist in a \textit{superposition} of multiple states until a measurement takes place. The space of possible states is described by a complex Hilbert space \(\mathcal{H} = \mathbb{C}^n \), with each of the \( n \) mutually orthonormal basis vectors \( \ket{a, b, c, \cdots} \) specifying it representing a distinct state. The information of a quantum system is captured in its \textit{state vector} \( \ket{\Psi} \in \mathbb{C}^n\). The state vector can be expanded in terms of the basis vectors,
\[
    \ket{\Psi} = \sum_{a,b,c,\cdots}\psi(a,b,c,\cdots)\ket{a,b,c,\cdots},
\]
with the coefficients \( \psi(a,b,c,\cdots) \) designating the \textit{wavefunction} of the system. The state vector is normalized, which places the restriction
\[
    \sum_{a,b,c,\cdots}\psi^*(a,b,c,\cdots)\psi(a,b,c,\cdots) = 1
\]
on the wavefunction.

Generally, the basis vectors of \( \mathcal{H} \) correspond to a definite state in which a particle can appear after measurement. The outcome of each observation is probabilistic, with the likelihood \( P \) to observe the state denoted by \( \lambda_i \) given by the Born rule:
\[
P(\lambda_i) = \langle \lambda_i | \Psi \rangle \langle \Psi | \lambda_i \rangle = \left| \langle \lambda_i | \Psi \rangle \right|^2.
\]

It follows that the state vector represents the probability distribution of a particle's measurement outcomes. In between measurements, the state vector, and likewise this distribution, evolves over time as described by the time-dependent Schrödinger equation:
\[
    \hbar \frac{\partial \ket{\Psi}}{\partial t} = -i \mathbf{H} \ket{\Psi}.
\]
Here, \( \mathbf{H} \) is a Hermitian operator known as the \textit{Hamiltonian}, which represents the energy of the system; its eigenvalues are the allowed energy levels of the system. The physical constant \( \hbar \) is the \textit{reduced Planck's constant}, with value \( \hbar \approx 1.055 \times 10^{-34} \text{ J} \cdot \text{s}\). Importantly, the state at any time \( t \) can be written as
\[
    \ket{\Psi(t)} = \mathbf{U}(t) \ket{\Psi(0)},
\]
with \( \ket{\Psi(0)} \) being the initial state of the system and \( \mathbf{U}(t) \) being a particular time-dependent unitary linear operator (satisfying \( U^\dagger(t) U(t) = I \)) known as the \textit{time-development operator}, which is found by solving the Schrödinger equation~\cite{susskind2013quantum,verma2009quantum}.

Notably, both Hamiltonians and unitary operators express the process of temporal evolution in QM. In \framework, we maintain the quantum relationships between Transformer analogs of these quantities to leverage the preexisting mathematical structure of QM in our exploration. This permits \framework to fluidly transition between synonymous quantum descriptions of an underlying physical behavior, unearthing twin quantum mechanical viewpoints of Transformers, each with their own distinct advantages. We examine this dualism in detail in sections~\ref{sec_analogy} and~\ref{sec_experiments}, tracing it from its theoretical roots to its practical application.

\subsection{The Tuned Lens}

While the evolution of quantum states can be viewed as the updating of a probability distribution, this concept does not translate directly to Transformers. This is because Transformer layers refine abstract hidden states, with a distribution typically only extracted in the last layer. To resolve this distinction, \framework employs the Tuned Lens~\cite{Belrose2023ElicitingLP}, a probe designed to translate intermediate hidden states into prediction logits. Normalizing these logits produces an output distribution, allowing the trajectory of a Transformer's output probabilities to be tracked. We review this tool in this section, beginning from its roots in the standard logit lens and exploring Belrose et al.'s augmentation of it.

Consider a Transformer with \( L \) layers, each operating in a latent space with dimensionality \( D \). Denote the hidden state of the token used by the final prediction head after \( 0 \leq \ell \leq L\) hidden layers of the model have been processed as \(h_\text{out}^{\ell +1} \in \mathbb{R}^D\).

The \(\ell\)-th layer of the model performs the update,
\[
    h_\text{out}^{\ell +1} = h_\text{out}^{\ell} + F_\ell(h_\text{out}^{\ell +1}),
\]
to this hidden state, where \( F_\ell(\cdot) \) is the residual update of the \(\ell\)-th layer.

Applying the residual update rule recursively, we observe that the final hidden state \(h_\text{out}^{L + 1}\) is given by
\[
    h_\text{out}^{L + 1} = h_\text{out}^{\ell} + \sum_{\ell^\prime = \ell}^L F_{\ell^\prime} (h_\text{out}^{\ell^\prime + 1}).
\]
The calculation of the logit vector \( z^L \) can then be formalized as follows:
\[
    z^L = \text{LayerNorm} \left ( h_\text{out}^{\ell} + \sum_{\ell^\prime = \ell}^L F_{\ell^\prime} (h_\text{out}^{\ell^\prime + 1}) \right ) W_U + b_U.
\]
Here, \( \sigma(\cdot) \), \( W_U \), and \( b_U \) denote the softmax function, unembedding matrix, and unembedding bias respectively.

The logit lens~\cite{nostalgebraist2020logitlens} computes the intermediate logit distribution before any layer \( \ell \), \( z^\ell \), by setting the residual update of future layers to zero:
\[
    z^\ell = \text{LayerNorm} ( h_\text{out}^{\ell}) W_U + b_U.
\]
Although this simple design works well for models such as GPT-2~\cite{Radford2019LanguageMA}, it fails to retain its effectiveness across model families~\cite{nostalgebraist2021LogitLensExtensions}. Furthermore, the logit lens suffers from \textit{bias} (it often places more probability mass on certain tokens than the prediction layer does) and \textit{representation drift} (where the latent space of intermediate layers may not align with that of the final layer).

To correct for these flaws, the Tuned Lens introduces a learned translator composed of a weight matrix \( A_\ell \in \mathbb{R}^{D \times D}\) and bias \(b_\ell \in \mathbb{R}^D \) to the logit lens formulation, obtaining \( z_\ell \) as follows:
\[
    z^\ell = \text{LayerNorm} ( h_\text{out}^{\ell}A_\ell + b_\ell) W_U + b_U.
\]
The two additions that the Tuned lens makes each account for a different shortfall of the original logit lens. The constant \( b_\ell \) debiases the lens by learning the average residual update, thereby correcting for cases where the Transformer's average updates significantly deviate from zero. The weights \( A_\ell \) function as a change-of-basis matrix that maps the intermediate latent space to that used by the final layer. After training to minimize KL-divergence against prediction logits yielded from the original Transformer, the Tuned Lens yields substantially lower bias and perplexity than the standard logit lens.

In \framework, we adopt the Tuned Lens as a foundation from which the model's latent probabilities can be examined. We note that the space of possible probabilities is constrained, primarily by the condition that applying softmax to logits ensures that they are normalized to sum to 1. This motivates the translation of the mathematics of quantum mechanics: its definition of state vectors and unitary operators renders it uniquely suited to describe transformations between probability distributions. We formally introduce the core components of \framework in the subsequent section.

\section{The \framework Framework} \label{sec_analogy}
To make the relationship between Transformers and QM concrete, we now formally introduce \framework's postulates of a quantum-inspired perspective of Transformers. Please note that this section is not intended to be interpreted as a direct isomorphism between the two subjects, but rather as a simple framing of the internals of Transformers in the language of QM to motive \framework's protocol. The core definitions of \framework mirror the fundamental principles of quantum mechanics and establish foundations for Hilbert spaces, state vectors, unitary operators, and Hamiltonians in the context of Transformers. By deriving these quantities for a pre-trained Transformer, insights regarding the action of its layers may be obtained.

Broadly, a state vector of a Transformer is a transformed representation of its internal probability distribution, whose expansion in the Hilbert basis reflects the component of the total probability mass afforded to each token. The model's state vector is successively modified by the layers of a Transformer, which are treated as unitary operators. Extending this structure yields a Hamiltonian associated with each layer that serves as the generator of an analogous residual stream output. Together, these postulates provide a novel end-to-end perspective of Transformer dynamics.

\subsection{Postulate 1: Hilbert Basis}
One of the most fundamental concepts in QM is the state space of a particle, formally categorized as a Hilbert space. In QM, each basis vector of this space specifies a distinct state of a particle, such as spin-up \( \ket{u} \) or spin-down \( \ket{d} \). Akin to a quantum system's confinement to a finite number of measurable outcomes, a Transformer is restricted to a delimited set of distinct outputs. For a Transformer, the group of unique possible outputs for a given input sequence is the set of output units \( C = \{ c_1, c_2, c_3, \dots, c_N \} \), with \( N \) representing the number of possible outputs for the model. Accordingly, a natural basis for a Transformer is the set of vectors \( \mathcal{C} = \{ \ket{c_1}, \ket{c_2}, \ket{c_3}, \dots, \ket{c_N}\} \), where each basis vector \( \ket{c_i} \in \mathbb{C}^N \) corresponds to the \(i\)-th token in the model's vocabulary.

\subsection{Postulate 2: Orthogonality of Basis Vectors}
QM delineates that unambiguously distinguishable states must be ascribed orthogonal vectors. As a Transformer cannot output two distinct output units at once, these states are discernible by the action of selecting from its probability distribution. Thus,

\begin{equation} \label{eq_orthogonality}
    \langle c_i | c_j \rangle = 
    \begin{cases} 
        0, & \text{for } i \neq j \\
        1, & \text{for } i = j
    \end{cases}
\end{equation}

It is also conventional for Hilbert bases to be normal. This can be expressed mathematically as the condition:

\begin{equation} \label{eq_normality}
    |\ket{ c_i} || = 1.
\end{equation}

To satisfy eqs.~\ref{eq_orthogonality} and~\ref{eq_normality} in conjunction, we elect to express the \(i\)-th basis vector as the following column vector:
\[
    \ket{c_i} = \begin{pmatrix}
            x_1 \\
            x_2 \\
            \vdots \\
            x_n
            \end{pmatrix} \quad \text{where } x_j = \delta_{ij},
\]
where \( \delta_{ij} \) refers to the Kronecker delta symbol. This form is similar to a one-hot encoding of each output unit and defines the basis \( \mathcal{C} \) as orthonormal. This insight will prove crucial in forthcoming derivations involving Transformer state vectors, which are introduced in the following postulate.

\subsection{Postulate 3: State Vectors} \label{subsec_state_vectors}

Transformers operate by iteratively transforming a sequence of token representations to produce an output state for each token which may be used for token-level or sequence-level predictions, depending on task specifics. Given an input sequence of length \( n \), \( I = \{i_1, i_2, i_3, \cdots, i_n \} \) such that \(\forall i \in I, i \in C\), a Transformer produces an over distribution \( \theta \) all elements in \( C \). In the Hilbert basis \( \mathcal{C} \), this distribution can be ascribed a corresponding state vector denoted as \( \ket{\Psi} \). This stems from the standard quantum formulation, where the notion of a state vector is intimately tied to the probabilities distribution of outcomes that would be observed if a measurement were to be conducted. To emulate this principle, \( \ket{\Psi} \) is defined such that
\begin{equation} \label{eq_4.3_1}
    P(c_i) = \langle c_i | \Psi \rangle \langle \Psi | c_i \rangle = \left| \langle c_i | \Psi \rangle \right|^2,
\end{equation}
where \( P(c_i) \) is the probability that the output unit \( c_i \) is generated by the Transformer. If \( h_{\text{out}} \in \mathbb{R}^D \) is the final hidden state (post layer norm) for the token used by the prediction head of a Transformer with a \( D \)-dimensional embedding space, then the logit vector over the entire vocabulary is obtained by a linear transformation of  \( h_{\text{out}} \). The probability \( P(c_i) \) is then the \( i \)-th component of the probability distribution resulting from applying the softmax function to this logit vector:
\begin{equation} \label{eq_4.3_2}
    P_{j+1}(c_i) = \sigma( h_{\text{out}} W_U + b_U)_i.
\end{equation}
Here, \(W_U \in \mathbb{R}^{D \times N}\) and \( b_U \in \mathbb{R}^N \) are the unembedding matrix and bias of the Transformer, and \( \sigma(\cdot) \) is the softmax function.

However, a similar process can also be executed after the completion of any intermediate layer in the model as well to track the evolution of its internal probability distribution during inference time~\cite{nostalgebraist2020logitlens,Belrose2023ElicitingLP}. Consider a Transformer with \( L \) hidden layers. Let \( h^\ell_{\text{out}} \in \mathbb{R}^D \) denote the hidden state of the token used by the prediction head after the outputs of \( \ell \) layers in the model architecture have been added to the residual stream. For clarity, \( h^0_{\text{out}} \) would be the hidden state post-embedding and positional encoding, \( h^L_{\text{out}} \) would be the final hidden state before unembedding. Then, likewise to eq.~\ref{eq_4.3_1} above, the hidden state \( h^\ell_{\text{out}} \) can be mapped to the model's vocabulary through a function \( f: \mathbb{R}^D \to \mathbb{R}^V \). In practice, this function can take the form of the logit lens~\cite{nostalgebraist2020logitlens} or one of its enhanced variants~\cite{Belrose2023ElicitingLP}. Thus, matching the structure of eq.~\ref{eq_4.3_2} above, the state vector \( \ket{\Psi^\ell} \) for this token after the processing of \(t\) layers is given by solving
\begin{equation} \label{eq_probs_from_lens}
    \left| \langle c_i | \Psi^\ell \rangle \right|^2 = P^\ell(c_i) = \sigma \left[ f(h^\ell_{\text{out}}) \right ]_i
\end{equation}
over all units in the output space, with \(P^\ell(c_i)\) referring to the probability of generating the unit \( c_i \) after \( \ell \) layers. This implies that \( \ket{\Psi^\ell} \) has the column form:
\begin{equation} \label{eq_state_vec_components}
    \ket{\Psi^\ell} = \begin{pmatrix}
                            \sqrt{P^\ell(c_1)} \\
                            \sqrt{P^\ell(c_2)} \\
                            \vdots \\
                            \sqrt{P^\ell(c_N)}
                            \end{pmatrix}
\end{equation}
Intuitively, this means that the \(k\)-th component of the state vector has a magnitude given by the square root of the probability of obtaining the \( k \)-th unit in the model's output space. A consequence of this is that
\begin{equation}\label{eq_model_state}
||\Psi^\ell|| = \sqrt{ \langle \Psi^\ell | \Psi^\ell \rangle } = 1,
\end{equation}
which corresponds to the fact that summing probabilities over all tokens yields one. The ket \( \ket{\Psi^\ell} \) is layer-dependent, and the mechanism for its evolution is introduced in the next postulate.

\subsection{Postulate 4: Layers and Schrödinger Dynamics} \label{subsec_layers_qc}
The primary function of a Transformer layer is to nudge its inputs towards outputs that better reflect the correct frequency distribution for later tokens in the relevant context. Building on the framework described thus far, we can now examine how an Transformer's internal state evolves during its forward pass. Employing the notation of eq.~\ref{eq_model_state}, the \(\ell\)-th hidden layer of an Transformer can be characterized as mapping the input product state vector \( \ket{\Psi^{\ell - 1}} \to \ket{\Psi^\ell} \), the resulting output. In QM, such state transitions are governed by the Schrödinger equation and are facilitated by unitary operators. Therefore, in this analogy we devise a unitary operator \( \mathbf{U}^\ell \in \mathbb{C}^{N \times N} \) to describe the action of the \( \ell \)-th Transformer layer on its input state vector:

\begin{equation} \label{eq_general_unitary_layer}
    \ket{\Psi^\ell} = \mathbf{U}^\ell \ket{\Psi^{\ell-1}} \quad \forall \ell \in \mathbb{Z}, 1 \leq \ell \leq L.
\end{equation}

For clarity, we highlight that the unitary operator \( \mathbf{U}^\ell \), and by extension the Hamiltonian \( \mathbf{H}^\ell \) (defined in the following postulate), is dependent on the layer's input. In general, a Transformer prediction trajectory will not be identical across dataset instances, and thus a single operator is unable to represent this complexity by itself. Instead of narrowing to one input, \framework computes operators across a broad set of examples to analyze holistic patterns as illustrated in section~\ref{sec_experiments}.

As a technical note, quantum mechanical unitary operators are generally functions of time, whereas to translate this concept to Transformers we remove this time dependence by considering a fixed time interval, as in common in domains such as quantum computing~\cite{nielsen}. By applying the operators corresponding to each hidden layer of a Transformer in succession, the composite state vector begins as \( \ket{\Psi^0} \) and travels through the Hilbert space \( \mathcal{C} \), gradually converging to its final state \( \ket{\Psi^L} \). 

We briefly remark that this process is similar to that behind quantum computing algorithms such as Grover's algorithm~\cite{Grover1996}, where a desired state vector is strategically manipulated by a combination of quantum gates to prepare a desired output distribution~\cite{nielsen,3B1B_Grovers_Algorithm}.

\subsection{Postulate 5: The Hamiltonian Lens} \label{subsec_Hamiltonian_lens}

By framing Transformer layers as unitary operators, we obtain access to additional tools from QM to investigate their action on state vectors. Specifically, we can now associate each layer's unitary \( \mathbf{U}^\ell \) with a Hamiltonian \( \mathbf{H}^\ell \). We will now discuss the properties of this Hamiltonian, demonstrating its connection with the change in the merged state vector \( \ket{\Psi^\ell} \) over a layer.

QM dictates that Hamiltonian generates its associated unitary in accordance with the following relation derived from solving the time-dependent Schrödinger equation~\cite{WilliamsExplorationsQC}:
\[
    \mathbf{U}(t) = \exp(\frac{-it\mathbf{H}}{\hbar}).
\]
In converting the above equation to our description of Tranformers, we substitute \( \mathbf{U}(t) \) and \( \mathbf{H} \) with our defined analogs of \( \mathbf{U}^\ell \) and \( \mathbf{H}^\ell \). In this exchange, we also replace the quantity \( \frac{t}{\hbar} \) with a coefficient \( \alpha \), which captures the relevant factors at play while generalizing beyond the constants of QM.

\begin{equation} \label{eq_Hamiltonian_def}
    \mathbf{U}^\ell = \exp \left ( {-i\alpha \mathbf{H}^\ell} \right ).
\end{equation}

As in QM, the unitary constraint on \( \mathbf{U}^\ell \) configures \( \mathbf{H}^\ell \) to be Hermitian. Therefore, the normalized eigenvectors \( \ket{E_j^\ell} \) of \( \mathbf{H}^\ell \) specify a basis to expand \( \ket{\Psi^\ell} \) in:

\begin{equation} \label{eq_Phi_energy_eigenvalue_decomp}
    \ket{\Psi^\ell} = \sum_{j} k_j\ket{E_j^\ell}
\end{equation}

Using eqs.~\ref{eq_general_unitary_layer},~\ref{eq_Hamiltonian_def}, and~\ref{eq_Phi_energy_eigenvalue_decomp}, we put forth the following theorem that illustrates that this Hamiltonian is intimately tied to the change in the model's state vector over the course of a layer.

\begin{theorem} \label{theorem_hamiltonian_delta_psi}
    Given a state vector \( \ket{\Psi^\ell} \) that passes through layer \(\ell\) with unitary \( \mathbf{U}^\ell \), the change \( \ket{ \Delta \Psi^\ell} \) during this process is determined by the eigenvectors and eigenvalues of the Hamiltonian \( \mathbf{H}^\ell \). Specifically,
    \[
        \ket{ \Delta \Psi^\ell} =  \sum_j k_j \left ( e^{ -i\alpha E_j } - 1 \right ) \ket{E_j^\ell}.
    \]
\end{theorem}

\begin{proof}
    Presented in Appendix A.1.
\end{proof}

Through Theorem~\ref{theorem_hamiltonian_delta_psi}, the residual addition of Transformers can be viewed through the lens of the Hamiltonian. Paralleling the residual update of a standard Transformer layer, the unitary formulation of \framework is equivalent to updating the model's state vector by an amount \( \ket{ \Delta \Psi^\ell} \) solely subject to the properties of the Hamiltonian. The eigenvalues of the Hamiltonian dictate the proportional change in the component of the state vector along each of the Hamiltonian's eigenvectors through the term \( \left ( e^{ -i\alpha E_j } - 1 \right ) \). Each \(E_j\) provides a phase angle for the modification of the state vector in a particular direction of the Hilbert space, and \( \alpha \) acts as a general layer-specific scaling factor. Therefore, while the unitary description of layers suggests compositional multiplication of the model's state vector, the Hamiltonian produces the same process through simple vector addition. By carrying out this process across all layers, the state vector shifts to approximate the true distribution for the subsequent tokens.

\section{Experiments} \label{sec_experiments}
We now operationalize the definitions of \framework by extracting state vectors, unitary operators, and Hamiltonians from a simple Transformers across three common deployment tasks: sentiment classification, item recommendation, and text generation. We implement models using the Sentihood~\cite{Saeidi2016Sentihood}, Amazon Books~\cite{Ni2019Justifying}, and Tiny Stories~\cite{Eldan2023TinyStoriesHS} datasets for these tasks, respectively. Below, we describe our experimental setup for the three datasets, followed by our results and interpretations of them.

\subsection{Experimental Setup}

In this section, we begin by establishing the general procedure guiding our analysis, followed by the specific modifications made to accommodate the task distinctions between datasets.

\subsubsection{Overarching Procedure} \label{subsubsec_gen_procedure}

For each dataset, we prepare a simple Transformer model with a single decoder block. Specifically, each model consists of an embedding segment that incorporates both token and positional encoding. These embeddings are then updated by an attention layer~\cite{Vaswani2017AttentionIA} before passing through a feed-forward network. Both the attention and feed-forward layers include post-layer normalization. The layer-normalized hidden state of the MLP layer is subsequently mapped to logits through a linear classification layer and softmax-normalized to obtain class probabilities. 

Upon this Transformer, we train two separate Tuned Lenses to derive intermediate probability distributions from the model. The \textit{embedding lens} is trained using the hidden states produced by the Transformer after the initial GPT-2 embeddings are compressed. Similarly, the \textit{attention lens} is conditioned upon the post-attention hidden states. These two lenses, combined with the output probability distribution produced by the model, allow us to track model prediction trajectories at three distinct stages in its inference.

Specifically, for each instance in the preprocessed dataset, we deternine the three aforementioned probability distributions and use them to extract three state vectors (eqs.~\ref{eq_probs_from_lens} and~\ref{eq_state_vec_components}), denoted \( \ket{\Psi^0} \), \( \ket{\Psi^1} \), and \( \ket{\Psi^2} \) in accordance with our previously established notation. We then construct appropriate unitary operators \( \mathbf{U}^1 \) and \( \mathbf{U}^2 \) to map between states \( \ket{\Psi^0} \) and \( \ket{\Psi^1} \) and states \( \ket{\Psi^1} \) and \( \ket{\Psi^2} \) respectively (eq. \ref{eq_general_unitary_layer}). Consequently, the operator \( \mathbf{U}^1 \) describes the action of the attention layer, while correspondingly \( \mathbf{U}^2 \) illuminates that of the feed-forward layer. We note that in general there exist an infinite number of possible unitary matrices that map from one state to another, and thus choose to constrain the form of the operators \( \mathbf{U}^1 \) and \( \mathbf{U}^2 \) to be Householder transformations~\cite{Householder1958Unitary} for their computational efficiency and simplicity. These transformations can be viewed as reflections across a plane, given by
\[
    I - 2\ket{v}\bra{v},
\]
where \( \ket{v} \) is the normal vector for the reflecting hyperplane. Alongside these unitary operators, we derive their associated Hamiltonians (eq.~\ref{eq_Hamiltonian_def}). These Hamiltonians provide a quantum-mechanical representation of residual updates, enabling \framework to compute the changes of the state vector between layers. To validate this approach, we put \framework into practice by analyzing the empirical data produced by it and presenting our interpretation.

\subsubsection{Dataset Descriptions and Implementation Details}

\textbf{Sentihood} Sentihood~\cite{Saeidi2016Sentihood} is set of 5212 annotated sentences taken from Yahoo! Answers posts reviewing urban neighborhoods in London. Each instance is labeled with one or more \textit{aspects}—specific characteristics of the location(s) mentioned in the accompanying sentence (e.g., transportation availability, safety). For each aspect, Sentihood includes the reference target location in the sentence and the sentiment (positive or negative) expressed towards that aspect.

We preprocess the dataset by removing all instances involving multiple locations or aspects, retaining only sentence texts and sentiment pairs. This leaves 1,864 instances, of which 1,329 convey positive sentiment and 535 indicate negative sentiment. We balance the dataset to an exact 1:1 ratio of positive and negative instances, yielding a final corpus of 1,070 instances (535 positive and 535 negative). This corpus is then divided into train and test sets in an 80:20 ratio.

On this preprocessed dataset, we train a simple Transformer model with a single decoder block to classify the sentiment of input texts to the model. Sentence tokenization and embedding is carried out by the GPT-2~\cite{Radford2019LanguageMA} tokenizer (extended to prepend a classification token to each instance) and its embedding and positional encoding matrices, producing 768-dimensional embeddings. Due to computational limitations, we compress these embeddings to a dimensionality of 128 via a learnable linear layer before passing through the remainder of the model. We train this Transformer for 10 epochs using binary cross-entropy loss against the sentiment labels, resulting in a test accuracy of 0.7664. 

Using this Transformer, we train the attention and embedding lenses (section~\ref{subsubsec_gen_procedure}). These lenses are instantiated as bias-only lenses without the translator weight matrix \( A_\ell \) due to the simplicity of this toy model. The resulting lenses provide the necessary components to compute the state vectors \( \ket{\Psi^0} \), \( \ket{\Psi^1} \), and \( \ket{\Psi^2} \) for the Sentihood corpus, allowing us to apply \framework to the model's intermediate decision process.

\textbf{Amazon Books} The Amazon Books corpus is a subset of the broader collection of Amazon Review Data~\cite{Ni2019Justifying}. The dataset contains over 51 million product purchase interaction between Amazon users, referencing over 3 million distinct books. Each interaction is further characterized by a product rating and timestamp; and each book is annotated with its title, content categories (e.g., books, music, etc.), brand, and sales rank.

We preprocess the dataset by reducing it to the top 50,000 most active users and filter the resulting data to be 5-core. Simultaneously, we also limit the maximum number of interactions per user to 200. Finally, we narrow the task to next-item prediction by retaining only the associated book title for each interaction. The ensuing dataset consists of 42,696 temporally-ordered user interaction histories spanning 185,136 unique items.

To prepare a subject Transformer model to interpret, we create train, validation, and test splits by temporarily partitioning the data. Causal inference is performed on all but the last two, the second to last, and last interaction in the train, validation, and test sets, respectively. On this data, we train a self-attentive sequential recommendation model~\cite{Kang2018SelfAttentiveSR}. To accelerate training, we use GPT-2~\cite{Radford2019LanguageMA} to initialize item embeddings to the mean token embedding vector across all tokens present in the book's title. These embeddings are linearly compressed and augmented by trainable positional encoding vectors before entering the model's central Transformer decoder block. We train the model for 15 and 20 epochs and compare performance on the validation set, obtaining superior performance with the latter number of epochs. The final model achieves a normalized discounted cumulative gain at 10 (NDCG@10) of 0.0473 and a mean reciprocal rank at 10 (MRR@10) of 0.0422 on the test set.

Due to the large vocabulary of this model, we train attention and embedding lenses that incorporate both a learnable change-of-basis matrix and bias. With a trained sequential recommender and its associated lenses now primed, their resulting probability distributions are used to generate state vectors across the test corpus. This prepares the Amazon Books model for the construction of the Householder unitary operators and their associated Hamiltonians.

\textbf{Tiny Stories} Originally constructed to investigate the ability for small language models to produce coherent text, the Tiny Stories dataset~\cite{Eldan2023TinyStoriesHS} consists of over 2.1 million short anecdotes that predominately contain a simple vocabulary that a young child could comprehend. Using this dataset, Eldan and Li (2023) train multiple Transformer models with under 10 million parameters, obtaining performance that rivals that of significantly larger models such as GPT-2 XL.

We use the one-decoder block variant of Eldan and Li's Transformer models as a subject model for \framework. This model integrates a GPT-NEO~\cite{Black2022GPTNeo} style architecture, reduced to use 1024-dimensional embeddings and to contain only one multi-headed self-attention and MLP layer each. To prepare for training Tuned Lenses using this model, we preprocess the Tiny Stories validation set by splicing each instance at randomized locations to create input texts of various lengths. This yields a corpus of 21,990 texts that we pass through the Transformer model to obtain last-token hidden states to train an embedding and attention lens on. The combination of the pre-trained model and the trained lenses successfully yields the state vectors (\(\ket{\Psi^0}, \ket{\Psi^1}, \ket{\Psi^2} \)). In the following section, these state vectors are interpreted using the Householder transformations and Hamiltonians of the \framework framework.

\subsection{Results and Analysis}

\begin{table*}[t!]
  \centering
  \fontsize{7pt}{8.4pt}\selectfont

  \begin{subtable}{\textwidth}
    \centering
    \caption{Sentihood Results}
    \label{table_sentihood_experiments}
    \begin{tabular}{| >{\centering\arraybackslash}m{1.5cm} |
                    >{\centering\arraybackslash}m{2cm} |
                    >{\centering\arraybackslash}m{2.25cm} |
                    >{\centering\arraybackslash}m{1.5cm} |
                    >{\centering\arraybackslash}m{1.5cm} | }
      \hline
      \textbf{Layer} &
      \textbf{Unitary Similarity (P value)} &
      \textbf{Hamiltonian Similarity (P value)} &
      \textbf{Average \( \ket{\Delta \Psi} \) Magnitude (P-value)} &
      \textbf{Inter-layer Correlation (P value)} \\
      \hline \hline
      Multi-headed Attention &
      0.89 (0.0001) &
      0.95 (0.0001) &
      0.1066 (0.0099) &
      \multirow{2}{*}[-5pt]{\centering 0.4312 (0.0001)} \\
      \cline{1-4}
      Multi-Layer Perceptron &
      0.63 (0.0001) &
      0.82 (0.0001) &
      0.0890 (0.0099) &
      \\ 
      \hline
    \end{tabular}
  \end{subtable}

  \vspace{0.5cm}

  \begin{subtable}{\textwidth}
    \centering
    \caption{Amazon Books Results}
    \label{table_books_experiments}
    \begin{tabular}{| >{\centering\arraybackslash}m{1.5cm} |
                    >{\centering\arraybackslash}m{2cm} |
                    >{\centering\arraybackslash}m{2.25cm} |
                    >{\centering\arraybackslash}m{1.25cm} |
                    >{\centering\arraybackslash}m{2cm} |
                    >{\centering\arraybackslash}m{1.5cm} |
                    >{\centering\arraybackslash}m{1.0cm} |
                    >{\centering\arraybackslash}m{1.5cm} | }
      \hline
      \textbf{Layer} &
      \textbf{Unitary Similarity (P value)} &
      \textbf{Hamiltonian Similarity (P value)} &
      \textbf{Householder Clusters} &
      \textbf{Householder Cluster Cohesion (P value)} &
      \textbf{Average \( \ket{\Delta \Psi} \) Magnitude (P-value)} &
      \textbf{\( \ket{\Delta \Psi} \) Clusters} &
      \textbf{Inter-layer Correlation (P value)} \\
      \hline \hline
      Multi-headed Attention &
      \( 1 - 2.1556 \times 10^{-5}\) (0.0001) &
      \( 2.3209 \times 10^{-3}\) (0.0001) &
      35 &
      0.6527 (0.0099) &
      0.1081 (0.0099) &
      35 &
      \multirow{2}{*}[-5pt]{\centering 0.9255 (0.0001)} \\
      \cline{1-7}
      Multi-Layer Perceptron &
      \( 1 - 2.1588 \times 10^{-5}\) (0.0001) &
      \( 8.1543 \times 10^{-4}\) (0.0001) &
      25 &
      0.6815 (0.0099) &
      0.0908 (0.0099) &
      30 &
      \\ 
      \hline
    \end{tabular}
  \end{subtable}

  \vspace{0.5cm}

  \begin{subtable}{\textwidth}
    \centering
    \caption{Tiny Stories Results}
    \label{table_stories_experiments}
    \begin{tabular}{| >{\centering\arraybackslash}m{1.5cm} |
                    >{\centering\arraybackslash}m{2cm} |
                    >{\centering\arraybackslash}m{2.25cm} |
                    >{\centering\arraybackslash}m{1.25cm} |
                    >{\centering\arraybackslash}m{2cm} |
                    >{\centering\arraybackslash}m{1.5cm} |
                    >{\centering\arraybackslash}m{1.0cm} |
                    >{\centering\arraybackslash}m{1.5cm} | }
      \hline
      \textbf{Layer} &
      \textbf{Unitary Similarity (P value)} &
      \textbf{Hamiltonian Similarity (P value)} &
      \textbf{Householder Clusters} &
      \textbf{Householder Cluster Cohesion (P value)} &
      \textbf{Average \( \ket{\Delta \Psi} \) Magnitude (P-value)} &
      \textbf{\( \ket{\Delta \Psi} \) Clusters} &
      \textbf{Inter-layer Correlation (P value)} \\
      \hline \hline
      Multi-headed Attention &
      \( 1 - 7.8402 \times 10^{-5}\) (0.0001) &
      \( 1.4924 \times 10^{-2}\) (0.0001) &
      25 &
      0.0948 (0.0099) &
      0.1747 (0.0099) &
      35 &
      \multirow{2}{*}[-7pt]{\centering 0.5990 (0.0001)} \\
      \cline{1-7}
      Multi-Layer Perceptron &
      \( 1 - 7.8768 \times 10^{-5}\) (0.0001) &
      \( 1.0322 \times 10^{-2}\) (0.0001) &
      35 &
      0.0514 (0.0099) &
      0.0837 (0.0099) &
      35 &
      \\ 
      \hline
    \end{tabular}
  \end{subtable}

  \caption{A summary of the numerical results extracted using \framework across the three datasets of Sentihood (Subtable~\ref{table_sentihood_experiments}), Amazon Books (Subtable~\ref{table_books_experiments}), and Tiny Stories (Subtable~\ref{table_stories_experiments}). P values, when applicable, are provided as parentheticals.}
  \label{table_results_summary}
\end{table*}

\begin{figure*}[t]
  \centering
  \renewcommand{\arraystretch}{0.78}

  \begin{tabular}{
    >{\raggedright\arraybackslash}m{0.5cm}
    >{\centering\arraybackslash}m{0.22\textwidth}
    >{\centering\arraybackslash}m{0.22\textwidth}
    >{\centering\arraybackslash}m{0.22\textwidth}
    >{\centering\arraybackslash}m{0.22\textwidth}
  }
    \toprule
    &
    \textbf{Attention Householder} &
    \textbf{Attention \( \ket{\Delta\Psi} \)} &
    \textbf{MLP Householder} &
    \textbf{MLP \( \ket{\Delta\Psi} \)}
    \\
    \midrule

    \rotatebox{90}{\textbf{Sentihood}} &
      \includegraphics[width=\linewidth]{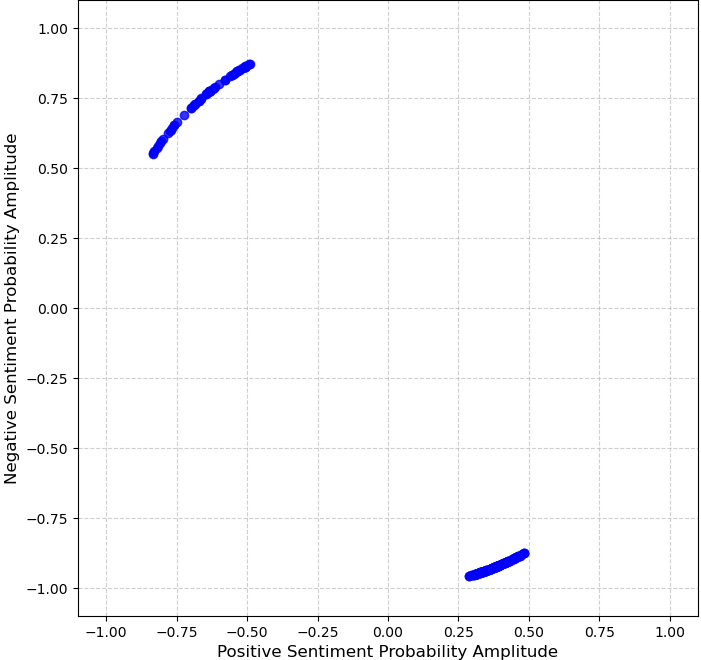} &
      \includegraphics[width=\linewidth]{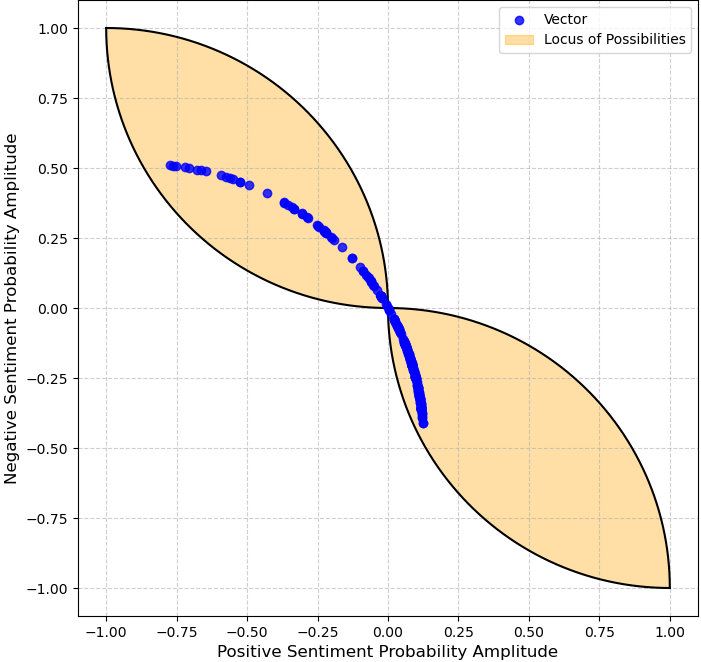} &
      \includegraphics[width=\linewidth]{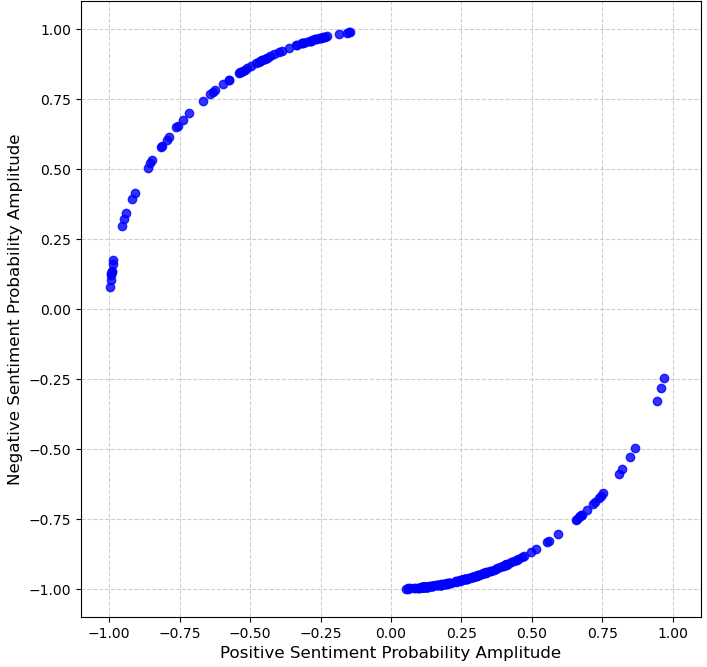} &
      \includegraphics[width=\linewidth]{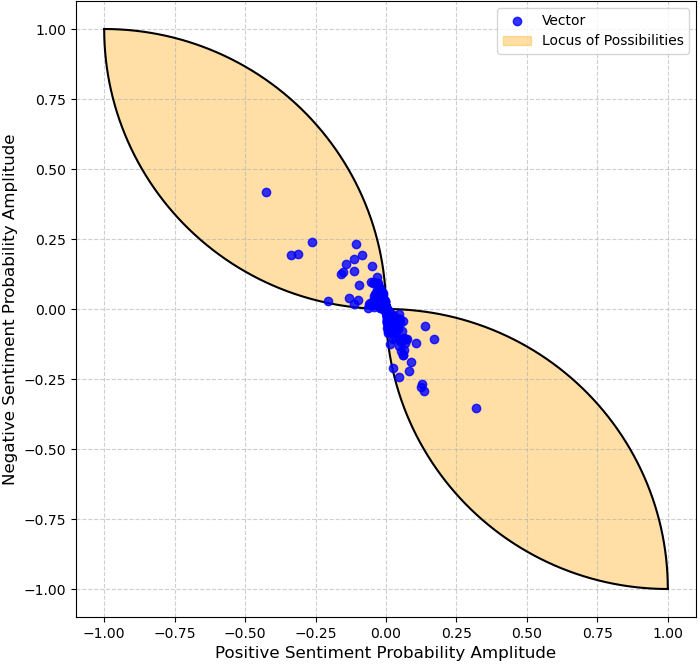}
      \\[0.4ex]

    \rotatebox{90}{\textbf{Amazon Books}} &
      \includegraphics[width=\linewidth]{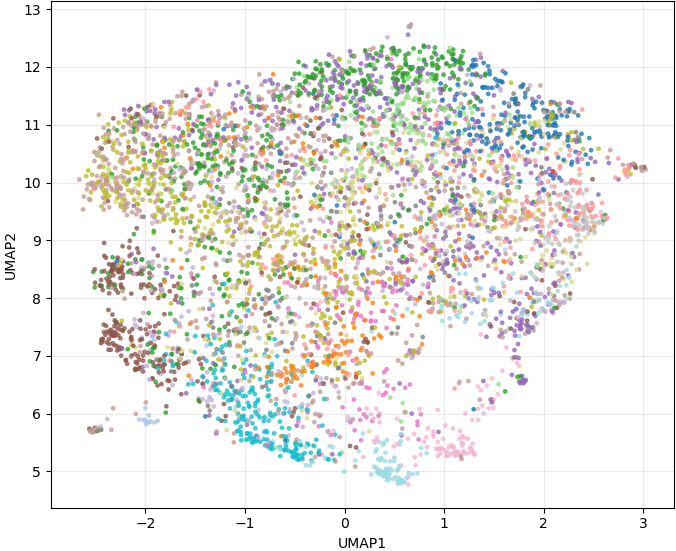} &
      \includegraphics[width=\linewidth]{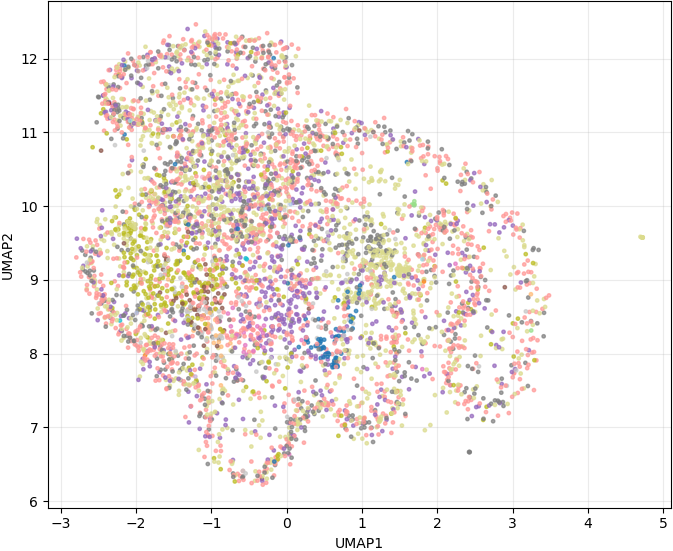} &
      \includegraphics[width=\linewidth]{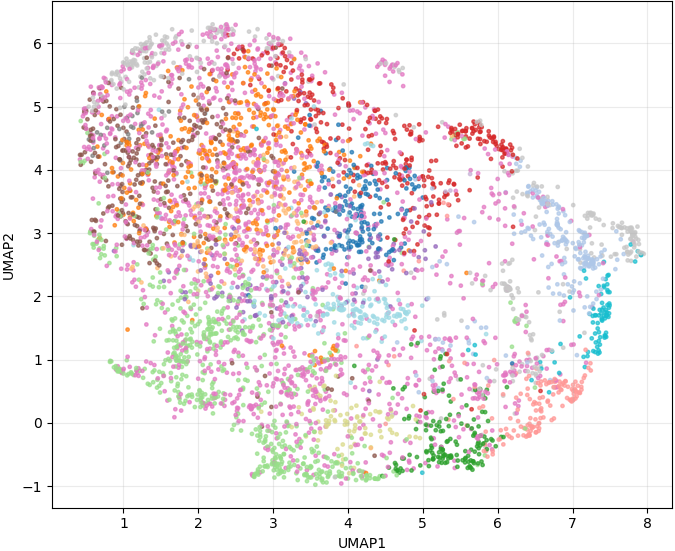} &
      \includegraphics[width=\linewidth]{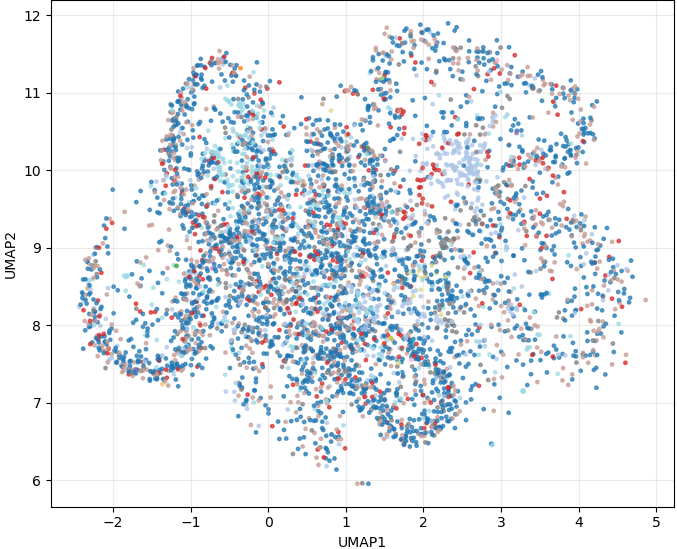}
      \\[0.4ex]

    \rotatebox{90}{\textbf{Tiny Stories}} &
      \includegraphics[width=\linewidth]{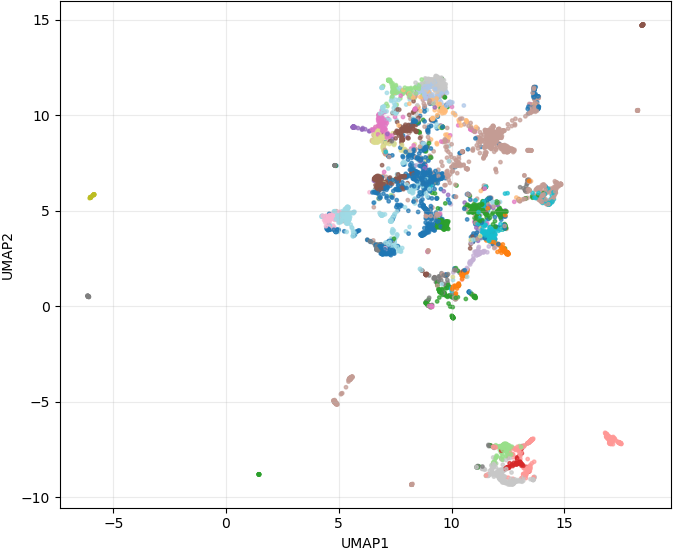} &
      \includegraphics[width=\linewidth]{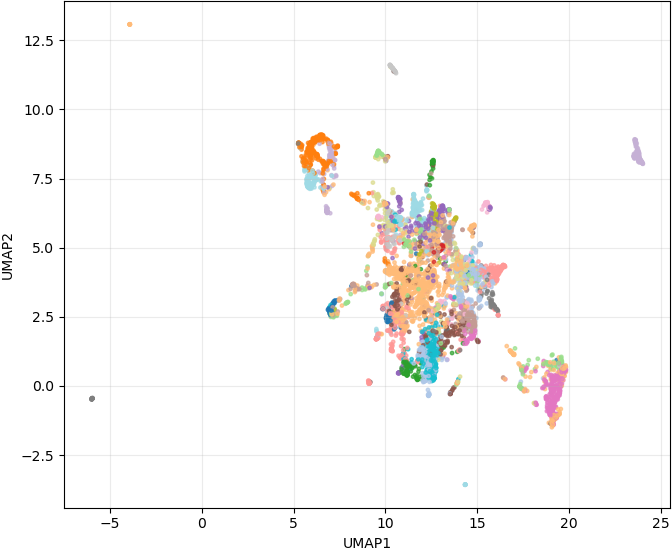} &
      \includegraphics[width=\linewidth]{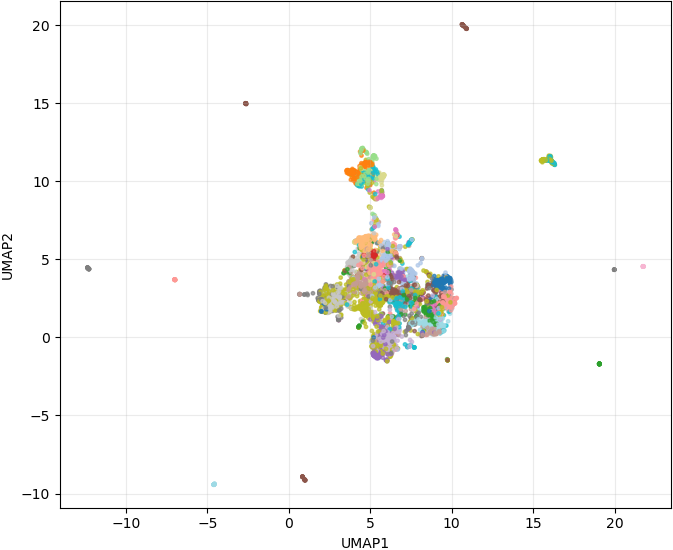} &
      \includegraphics[width=\linewidth]{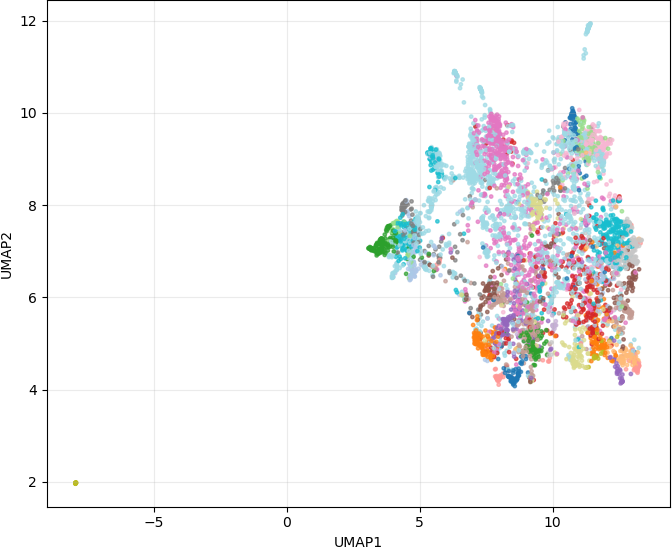}
      \\

    \bottomrule
  \end{tabular}

  \caption{Visualizations derived using \framework. Scatterplots are used for Sentihood, as are UMAP plots for the Amazon Books and Tiny Stories datasets.}
  \label{fig_dataset_plots}
\end{figure*}

We analyze each of the quantum-inspired quantities obtained by \framework from each dataset in succession. The mean pairwise Frobenius similarity across unitary operators in each dataset as a means of examining the cohesion of these operators. Likewise, a similar analysis is conducted on Hamiltonians. We then delve deeper into the Householder reflection vectors defining these two sets of operators, plotting and clustering them. To investigate inter-layer relationships, we compute distance correlation metrics between the attention and MLP layers across all three datasets' models. Finally, we apply Theorem~\ref{theorem_hamiltonian_delta_psi} to investigate \( \ket{\Delta \Psi} \) vectors. Figure~\ref{table_results_summary} summarize our numerical results. Overall, our analysis reveals that the quantum-inspired quantities extracted from all three models exhibit statistically significant non-random cohesion and structured state transitions, indicating a constrained, partially concept-driven layer update manifold.

\subsubsection{Unitary Operator Cohesion}

To search for trends in the action of the multi-head attention and MLP layers, we examine the distribution of pairwise Frobenius cosine similarities between unitary operators extracted from dataset instances. We compute the mean pairwise similarity, and evaluate its statistical significance by conducting two-sample permutation tests for each layer with the alternative hypothesis that the intra-layer unitary operator similarity is greater than that of a randomized control. Specifically, the unitary operators derived from each layer were pooled with those of the baseline before being randomly split into two groups whose intra-group similarity metrics were compared against the observed values. Add-one smoothing is further applied to account for data sparsity. Across all three datasets, the attention and MLP layers exhibit higher mean pairwise similarities between unitary operators than those of the randomly initialized control (Table~\ref{table_results_summary}). This implies that these unitary operators are not uniformly distributed; rather, the layers they stem from perform transformations that exhibit non-random consistency across inputs.

\subsubsection{Hamiltonian Cohesion}

Similarly, we test the mean pairwise similarity of the Hamiltonians with permutations tests designed analogously to those for unitary operations. This separate analysis helps confirm whether any consistency observed in transformations \( \mathbf{U^\ell} \) is rooted in a consistent generator of change \( \mathbf{H^\ell} \) across different inputs. Permutation tests designed analogously to those for unitary operations likewise produce p values of 0.0001 (Table~\ref{table_results_summary}) across all datasets and their specialized model's layers, providing convincing statistical evidence that these similarities in Hamiltonians are significantly greater than the randomized baseline.

\subsubsection{Householder Vector Analysis}
We further shed light on the unitary and Hamiltonian similarities by plotting the normalized Householder vectors characterizing the unitary operators and Hamiltonians. For the Sentihood dataset, a simple scatterplot is used as the output space is constrained to two dimensions. In the attention layer, the resulting Householder vectors form two concentrated clusters within sparse angular intervals (Figure~\ref{fig_dataset_plots}). Due to eq.~\ref{eq_state_vec_components}, Sentihood model state vectors occupy quadrant~\RomanNumeralCaps{1} in this plot. Consequently, the space of all Householder vectors is the union of arcs of the unit circle in quadrants~\RomanNumeralCaps{2} and~\RomanNumeralCaps{4}. Thus, this dense clustering suggests that not all probability updates are realized by the attention layer; instead, it relies on a constrained portion of all valid updates in its prediction trajectory. Conducting a similar analysis on the MLP layer, we find a greater spread in its Householder vectors (Figure~\ref{fig_dataset_plots}). This variance explains the comparitively lower similarity of both unitary operators and Hamiltonians in the MLP layer, as these quantum quantities are mathematically linked to the Householder vector. Accordingly, the MLP may be capable of realizing a wider variety of probability updates and state transformations than the attention layer.

In contrast with Sentihood, the Amazon Books and Tiny Stories datasets produce high-dimensional (50,000 or more dimensions) Householder vectors. Accordingly, a uniform manifold approximation and projection (UMAP)~\cite{McInnes2018UMAP} is used to visualize them in lower-dimensions while preserving their global structure. We cluster these two datasets' vectors with the k-means algorithm, optimizing k using elbow plots and coloring UMAP points to reflect their assigned cluster. Across both datasets, high-dimensional structure arises with different clusters separating into roughly delineated regions on the UMAP plots. Closer proximity between Householder vectors is observed in the MLP layer, which is particularly pronounced in the Tiny Stories dataset. To test whether these clusters may correspond to reasoning directions, we compute the mean vector of each cluster and examine their top 10 largest components. These components correspond to output units experiencing the largest probability gain over a layer. We search for alignment between these units by computing the mean pairwise cosine similarity between their embedding vectors in the original Transformer. As previous work has established the ability of these vectors to represent concepts~\cite{Mikolov2013EfficientEO}, this measures the conceptual alignment of the transformations emphasized by the layer. We then compare the mean cosine similarity of the clusters obtained via k-means against a randomized clustering approach with one-sided permutation tests to determine the former's statistical significance. In all cases, we observe p values under 0.01 (Figure~\ref{table_results_summary}), indicating that the previously observed cohesion in the transformations executed by these model's layers are at least partially concept-driven.

\subsubsection{Inter-layer Correlation}
By considering Householder vectors yielded by the attention and MLP layers from a single input, an inquiry into the relatedness of the transformations can be carried out. Specifically, we compute the distance correlation~\cite{Szekely2007MeasuringAT,Ramos2023Dcor}, a measure of dependence that captures both linear and non-linear relationships, between these pairs of Householder vectors on samples of all three datasets, finding moderately strong to strong relationships between them (Figure~\ref{table_results_summary}). Permutation tests with a null hypothesis of independence demonstrate that all values are statistically significant. Accordingly, the state transformations performed by downstream model layers is moderately dependent on those that come before it.

\subsubsection{State Vector Change Analysis}
Employing Theorem~\ref{theorem_hamiltonian_delta_psi}, we extract state vector changes and plot results. We again construct a scatterplot for the two-dimensional \( \ket{\Delta \Psi} \) vectors generated by \framework on Senithood. Consistent with the tight clustering of its Householder vectors, all attention layer \( \ket{\Delta \Psi} \) vectors lie on a limited arc. This is despite the locus of all possible changes in state vectors spanning the lens region in yellow (see Appendix A2 for proof). Intriguingly, the arc is asymmetric, suggesting the existence of a nonzero average change in state vector across the attention layer. Computing this mean vector, we find it has a magnitude of 0.1066: one that is found to be statistically significant following a permutation test against a randomized control (Figure~\ref{table_results_summary}). This implies that this attention layer is biased despite being trained on a balanced corpus.

A similar investigation on the MLP layer finds that its vectors \( \ket{\Delta \Psi} \) cluster more tightly around the origin (Figure~\ref{fig_dataset_plots}). Although the MLP layer's Householder vectors imply that it modifies state vectors along a more diverse collection of directions than the attention layer, it contributes less to the Transformer's final predictions, as corroborated by the reduced mean \( \ket{\Delta \Psi} \) magnitude (0.0890). This agrees with previous literature showing that earlier layers have a larger impact on downstream outcomes~\cite{Sajjad2020OnTE}.

Turning to the Amazon Books and Tiny Stories datasets, their \( \ket{\Delta \Psi} \) vectors are visualized via UMAP plots. From the Amazon Books corpus, diffuse, contiguous manifolds emerge. The largest clusters dominate,accounting for 53.28\% and 66.62\% of the attention and MLP layers' vectors, respectively. However, sparse pockets emerge between these clusters, illustrating that as in Sentihood not all valid \( \ket{\Delta \Psi} \) vectors are realized by the model. Contrastingly, the Tiny Stories' \( \ket{\Delta \Psi} \) vectors cluster far more densely. However, sporadic 'islands' of vectors form away from the central grouping of vectors. Therefore, most transformations inside this Transformer are similar, with a handful of inputs triggering particularly distinct updates.

While these UMAP plots provide a qualitative view, a detailed analysis requires examining central tendency. Accordingly, we compute the average \( \ket{\Delta \Psi} \) vector across all layers of the Amazon Books and Tiny Stories Transformers. We find all layers exhibit a statistically significant nonzero bias, indicating a consistent directional change to their input states. This is most pronounced in the Tiny Stories attention layer, whose average \( \ket{\Delta \Psi} \) has an L2 norm of 0.1747. We hypothesize this may relate to the prevalence of tokens in the training corpus, and test this by extracting the top 10 token-basis components of the mean \( \ket{\Delta \Psi} \). However, this is largely not the case, as uncommon tokens such as ' Thinking' and 'Value' appear in the top 5. This suggests that although final Transformer outputs align with a desired statistical distribution, intermediate layers may retain misaligned representations that are resolved by subsequent components.

\section{Limitations and Open Research Directions}
\framework provides a novel framework for interpreting Transformers from the perspective of quantum mechanics. However, as is common with interdisciplinary work, the current perspective raises multiple intriguing challenges and open questions that warrant future investigation. We acknowledge these considerations here to guide future research.

A significant difference between QM and Transformers is that QM and its associated operators are structurally linear~\cite{susskind2013quantum}, whereas Transformers derive many of their capabilities from their nonlinear components and activation functions~\cite{Agarap2018DeepLU,Bridle1989,Hendrycks2016GaussianEL,Ramachandran2017SwishAS}. This challenge of nonlinearities has been previously observed in work on quantum machine learning algorithms~\cite{Li2022QuantumSN}. In this paper, we reconcile these distinctions by conducting our initial analysis with \framework at a level of abstraction above intra-layer processes. Nevertheless, comprehensively conquering the nonlinearities involved and expanding \framework to model the intra-layer processes of MLPs and Attention blocks remain open challenges.

Furthermore, we note that the experimental inferences that emerged from the proof of concept of \framework may not generalize to other Transformers due to the diversity of model target tasks and sizes. However, the technique introduced by \framework is flexible, allowing it to in principle be extended to pretrained Transformers across scales and task specifics. We intend to continue to explore this area with supplemental studies on the capabilities of \framework.

In light of these gaps and opportunities, we present this paper as a foundational exploration into the inter-applicability of these disciplines. We hope that the mathematical language introduced in this work lays the groundwork for future studies to lengthen ties between Transformers and QM that lead to a unique understanding and analyzes of the former.

\section{Conclusion}
In this paper, we introduce \framework, a framework that charts a new path towards a quantum-inspired perspective of Transformers. By identifying conceptual similarities between QM and transformers, we developed a set of analogs for the core postulates of quantum mechanics and employed them to recast the standard forward pass through a language Transformer. Specifically, we exhibit that latent activations and hidden layers can be formally mapped to state vectors and unitary operators, respectively, through a proof-of-concept demonstration. This contribution opens several promising research directions for a quantum-based theoretical understanding of Transformers. Our work could be extended to integrate further components of quantum mechanics into this formalization, including entanglement to explore associations between layers and the general uncertainty principle for describing inherent inference tradeoffs, among others. A particularly noteworthy consequence of this research is the possibility to explore whether these analogies could be used to create new metrics for interpreting and evaluating Transformers. For instance, studies could adapt quantum information-theoretic measures to gauge the influence of inputs on generative models outputs, offering an alternative viewpoint on contextual dependence in Transformer inference. We intend to continue to explore this area by scaling the implementation of \framework to larger model sizes and investigating methods to address the critical nonlinearities in Transformers through the lenses of quantum channels and nonlinear Schrödinger equations.

\clearpage

\appendix
\section{Appendix}
\subsection{Proof of Theorem 1: Deriving Changes in State Vectors from Hamiltonians}

\begin{proof}
    The change \( \ket{ \Delta \Psi^\ell} \) can be written in terms of input and output state vectors of the layer \(\ell\):
    \[
        \ket{ \Delta \Psi^\ell} = \ket{\Psi^{\ell+1}} - \ket{\Psi^\ell}.
    \]
    Using eq.~\ref{eq_general_unitary_layer} we rewrite \( \ket{\Psi^{\ell+1}} \) in terms of \( \ket{\Psi^\ell} \):
    \[
        \ket{ \Delta \Psi^\ell} = \mathbf{U}^\ell \ket{\Psi^{\ell}} - \ket{\Psi^\ell}.
    \]
    Substituting the results of eqs.~\ref{eq_Hamiltonian_def} and~\ref{eq_Phi_energy_eigenvalue_decomp} we find:
    \[
        \ket{ \Delta \Psi^\ell} = \left [ \exp \left ( {-i\alpha \mathbf{H}^\ell} \right ) \sum_{j} k_j\ket{E_j^\ell} \right ] - \sum_{j} k_j\ket{E_j\ell}.
    \]
    Applying the Taylor series expansion of the matrix exponential and reorganizing the summations we obtain:
    \[
        \ket{ \Delta \Psi^\ell} = \left [ \sum_{m = 0}^\infty \sum_{j} { \frac{ \left ( -i\alpha \mathbf{H}^\ell \right )^m }{m!}} k_j\ket{E_j^\ell} \right ] - \sum_{j} c_j\ket{E_j^\ell}.
    \]
    Since \( \ket{E_j^\ell} \) is an eigenvector of \( \mathbf{H}^\ell \), we have \( \left ( \mathbf{H}^\ell \right )^m \ket{E_j^\ell} = (E_j)^m\ket{E_j^\ell} \). This allows us to put the above equation in the following form:
    \[
        \ket{ \Delta \Psi^\ell} = \left [ \sum_{m = 0}^\infty \sum_{j} { \frac{ \left ( -i\alpha E_j \right )^m }{m!}} k_j\ket{E_j^\ell} \right ] - \sum_{j} k_j\ket{E_j^\ell}.
    \]
    We now recondense the first summation by applying the earlier Taylor series relationship in reverse:
    \[
        \ket{ \Delta \Psi^\ell} =  \sum_j k_j e^{ -i\alpha E_j } \ket{E_j^\ell} - \sum_{j} k_j\ket{E_j^\ell}.
    \]
    Factoring the right side of the expression into a simpler form, we procure a simplified equation for \( \ket{ \Delta \Psi^\ell} \):
    \[
        \ket{ \Delta \Psi^\ell} =  \sum_j k_j \left ( e^{ -i\alpha E_j } - 1 \right ) \ket{E_j^\ell}.
    \]
    This satisfies the conclusion of the theorem and we are done. 
\end{proof}

\subsection{The Locus of State Vectors Updates}

We begin with observation that the postulate that establishes state vectors for Transformers constrains them to posses only positive real components in \( \mathcal{C} \). In particular, by eq.~\ref{eq_state_vec_components}, any state vector \( \ket{\Psi^\ell} \) lies on the surface of the unit \( N \)-dimensional hypersphere in the positive octant. The \(k\)-th component of \( \ket{\Psi^\ell} \), \( \ket{\Psi^\ell}_k \), is given by
\[
    \ket{\Psi^\ell}_k = \begin{cases}
                            \cos(\theta_{k})\prod_i^{k-1}\sin(\theta_i) & \text{ if } 1 \leq k \leq N-1 \\
                            \prod_i^{N - 1}\sin(\theta_i) & \text{ if } k = N
                        \end{cases}
\]
where \(\theta_1, \theta_2, \theta_3, \dots, \theta_N \in [0, \frac{\pi}{2}] \). Similarly, the updated state vector produced by the \(\ell\)-th layer, \( \ket{\Psi^{\ell + 1}} \) can be parameterized by \(  \phi_1, \phi_2, \phi_3, \dots, \phi_N \in [0, \frac{\pi}{2}] \):
\[
     \ket{\Psi^{\ell + 1}}_k = \begin{cases}
                                    \cos(\phi_{k}) \prod_i^{k-1} \sin(\phi_i) & \text{ if } 1 \leq k\leq N-1 \\
                                    \prod_i^{N - 1}\sin(\phi_i) & \text{ if } k = N
                                \end{cases}
\]
Thus, a vector \(\ket{\Delta \Psi^{\ell}}\) is an element of the locus of difference vectors \( \mathcal{S} \) if it satisfies the relation \( \ket{\Delta \Psi^{\ell}} = \ket{\Psi^{\ell + 1}} - \ket{\Psi^{\ell}}\), or its more descriptive equivalent:
\[
\begin{split}
& \ket{\Delta \Psi^{\ell}}_k = \\
&    \begin{cases}
            \cos(\phi_k) \prod_{i}^{k-1} \sin(\phi_i) - \cos(\theta_k)\prod_{i}^{k-1} \sin(\theta_i)
        & \text{ if } 1 \leq k\leq N-1 \\[1.5ex]
        \prod_{i}^{N - 1}\sin(\phi_i) - \prod_{i}^{N - 1}\sin(\theta_i) & \text{ if } k = N
    \end{cases},
\end{split}
\]
where each \( \theta_i, \phi_i \in [0, \frac{\pi}{2}] \). For the case where \(N = 2\), as in the Sentihood model, the restriction on \(\ket{\Delta \Psi^{\ell}}_k\) reduces to
\[
    \ket{\Delta \Psi^{\ell}}_k = \begin{cases}
                                    \cos(\phi) - \cos(\theta)  & \text{ if } k = 1 \\
                                    \sin(\phi) - \sin(\theta) & \text{ if } k = 2
                                \end{cases}
\]
Rewriting \( \ket{\Delta \Psi^{\ell}} \) in column form and applying the sum-to-product trigonometric identities we obtain:
\[
    \ket{\Delta \Psi^{\ell}} = \begin{pmatrix}
                                    -2\sin(\frac{\phi - \theta}{2})\sin(\frac{\phi + \theta}{2}) \\
                                    2\sin(\frac{\phi - \theta}{2})\cos(\frac{\phi + \theta}{2})
                                \end{pmatrix}.
\]
We define \( \delta = \phi - \theta \) and \( \epsilon = \phi + \theta \) and simplify:
\[
    \ket{\Delta \Psi^{\ell}} = 2\sin(\frac{\delta}{2})\begin{pmatrix}
                                    -\sin(\frac{\epsilon}{2}) \\
                                    \cos(\frac{\epsilon}{2})
                                \end{pmatrix} \quad \text{where } \delta \in [-\frac{\pi}{2}, \frac{\pi}{2}] \text{ and } \epsilon \in [0, \pi]
\]
While this is a complete mathematical description of the elements of \( \mathcal{S} \), a more intuitive understanding can be acquired by examining extreme cases. Consider fixing \( \theta = 0 \) while allowing \( \delta \) to vary. This yields the arc
\[
    \begin{pmatrix}
        \cos(\phi) - 1 \\
        \sin(\phi) - \sin(\theta)
    \end{pmatrix}
\]
which is equivalent to the quarter circle traced by the graph of \( (x + 1)^2 + y^2 = 1\) from \( (0, 0) \) to \( (-1, 1) \). Likewise, if we fix \( \phi = \frac{\pi}{2} \), then we obtain the arc formed by \( x^2 + (y - 1)^2 = 1\) between the same two points. These two quarter-circles bound the locus \( \mathcal{S} \) in quadrant~\RomanNumeralCaps{2}; the intersection of these disks is a \( \mathcal{R} \) comprising half of \( S \), with
\[
    \mathcal{R} = \{ (a, b): (a + 1)^2 + b^2 \leq 1 \text{ and } a^2 + (b + 1)^2 \leq 1 \}.
\]
Continuing this process by separately fixing \( \theta = \frac{\pi}{2} \) and \( \phi = 0 \) yields a similar lobe \( \mathcal{T} \) in quadrant~\RomanNumeralCaps{4}, given by:
\[
    \mathcal{T} = \{ (a, b): (a - 1)^2 + b^2 \leq 1 \text{ and } a^2 + (b - 1)^2 \leq 1 \}.
\]
Therefore, the complete locus \( \mathcal{S} \) is the union of these two components:
\[
    \mathcal{S} = \mathcal{R} \cup \mathcal{T}.
\]
This is the locus \( \mathcal{S} \) depicted in the Sentihood Attention and MLP \( \ket{\Delta \Psi} \) plots (Figure~\ref{fig_dataset_plots}).

\bibliographystyle{named}
\bibliography{references}

@article{blackbody_radiation,
  title={The Thermal Radiation Formula of Planck (1900)},
  author={Luis J. Boya},
  journal={ArXiv},
  year={2004},
  volume={abs/physics/0402064},
  url={https://api.semanticscholar.org/CorpusID:119414921}
}

@article{bohr_model,
  title={Revisiting Bohr’s Quantization Hypothesis for the Atomic Orbitals},
  author={J.H.O.Sales and A.T. Suzuki and D.S. Bonafé},
  journal={ArXiv},
  year={2006},
  volume={abs/physics/0608102},
  url={https://api.semanticscholar.org/CorpusID:118654048}
}

@book{susskind2013quantum,
  title     = {Quantum Mechanics: The Theoretical Minimum},
  author    = {Leonard Susskind and Art Friedman},
  year      = {2013},
  publisher = {Basic Books},
  address   = {New York},
  isbn      = {9780465062904}
}

@misc{nostalgebraist2020LogitLens,
  author = {nostalgebraist},
  title = {Interpreting GPT: the Logit Lens},
  url = {https://www.lesswrong.com/posts/AcKRB8wDpdaN6v6ru/interpreting-gpt-the-logit-lens},
  urldate = {2025-07-13},
  date = {2020-12-14},
  year = {2020},
  publisher = {LessWrong},
  addendum = {Blog post},
}

@misc{nostalgebraist2021LogitLensExtensions,
  author = {nostalgebraist},
  title = {logit lens on non-gpt2 models + extensions},
  url = {https://colab.research.google.com/drive/1MjdfK2srcerLrAJDRaJQKO0sUiZ-hQtA},
  urldate = {2025-10-5},
  year = {2021}
}

@article{Hopfield1982,
  author    = {Hopfield, J. J.},
  title     = {Neural networks and physical systems with emergent collective computational abilities},
  journal   = {Proceedings of the National Academy of Sciences},
  volume    = {79},
  number    = {8},
  pages     = {2554--2558},
  year      = {1982},
  publisher = {National Academy of Sciences},
  doi       = {10.1073/pnas.79.8.2554},
  pmid      = {6953413},
  pmcid     = {PMC346238}
}

@article{Ackley1985,
  author    = {David H. Ackley and Geoffery E. Hinton},
  title     = {A Learning Algorithm for Boltzmann Machines},
  journal   = {Cognitive science},
  volume    = {9},
  number    = {8},
  pages     = {147--169},
  year      = {1985},
}

@inproceedings{Bridle1989,
 author = {Bridle, John S},
 booktitle = {Advances in Neural Information Processing Systems},
 editor = {D. Touretzky},
 publisher = {Morgan-Kaufmann},
 title = {Training Stochastic Model Recognition Algorithms as Networks can Lead to Maximum Mutual Information Estimation of Parameters},
 url = {https://proceedings.neurips.cc/paper_files/paper/1989/file/0336dcbab05b9d5ad24f4333c7658a0e-Paper.pdf},
 volume = {2},
 year = {1989}
}

@article{BETTI2016,
title = {The principle of least cognitive action},
journal = {Theoretical Computer Science},
volume = {633},
pages = {83-99},
year = {2016},
note = {Biologically Inspired Processes in Neural Computation},
issn = {0304-3975},
doi = {https://doi.org/10.1016/j.tcs.2015.06.042},
url = {https://www.sciencedirect.com/science/article/pii/S0304397515005526},
author = {Alessandro Betti and Marco Gori}
}

@article{Liu2025NeuralTL,
  title={Neural Thermodynamic Laws for Large Language Model Training},
  author={Ziming Liu and Yizhou Liu and Jeff Gore and Max Tegmark},
  journal={ArXiv},
  year={2025},
  volume={abs/2505.10559},
  url={https://api.semanticscholar.org/CorpusID:278636268}
}

@inproceedings{Grover1996,
  title={A fast quantum mechanical algorithm for database search},
  author={Lov K. Grover},
  booktitle={Symposium on the Theory of Computing},
  year={1996},
  url={https://api.semanticscholar.org/CorpusID:207198067}
}

@article{Hendrycks2016GaussianEL,
  title={Gaussian Error Linear Units (GELUs)},
  author={Dan Hendrycks and Kevin Gimpel},
  journal={arXiv: Learning},
  year={2016},
  url={https://api.semanticscholar.org/CorpusID:125617073}
}

@article{Agarap2018DeepLU,
  title={Deep Learning using Rectified Linear Units (ReLU)},
  author={Abien Fred Agarap},
  journal={ArXiv},
  year={2018},
  volume={abs/1803.08375},
  url={https://api.semanticscholar.org/CorpusID:4090379}
}

@article{Ramachandran2017SwishAS,
  title={Swish: a Self-Gated Activation Function},
  author={Prajit Ramachandran and Barret Zoph and Quoc V. Le},
  journal={arXiv: Neural and Evolutionary Computing},
  year={2017},
  url={https://api.semanticscholar.org/CorpusID:196158220}
}

@book{nielsen,
  author = {Nielsen, Michael A. and Chuang, Isaac L.},
  publisher = {Cambridge University Press},
  title = {Quantum Computation and Quantum Information},
  year = 2000
}

@inproceedings{Vaswani2017AttentionIA,
  title={Attention is All you Need},
  author={Ashish Vaswani and Noam M. Shazeer and Niki Parmar and Jakob Uszkoreit and Llion Jones and Aidan N. Gomez and Lukasz Kaiser and Illia Polosukhin},
  booktitle={Neural Information Processing Systems},
  year={2017},
  url={https://api.semanticscholar.org/CorpusID:13756489}
}

@inproceedings{Radford2019LanguageMA,
  title={Language Models are Unsupervised Multitask Learners},
  author={Alec Radford and Jeff Wu and Rewon Child and David Luan and Dario Amodei and Ilya Sutskever},
  year={2019},
  url={https://api.semanticscholar.org/CorpusID:160025533}
}

@book{WilliamsExplorationsQC,
author = {Williams, Colin P. and Clearwater, Scott H.},
title = {Explorations in Quantum Computing},
year = {1997},
isbn = {038794768X},
publisher = {Springer-Verlag TELOS},
address = {Santa Clara, CA, USA}
}

@inproceedings{Mikolov2013EfficientEO,
  title={Efficient Estimation of Word Representations in Vector Space},
  author={Tomas Mikolov and Kai Chen and Gregory S. Corrado and Jeffrey Dean},
  booktitle={International Conference on Learning Representations},
  year={2013},
  url={https://api.semanticscholar.org/CorpusID:5959482}
}

@article{Belrose2023ElicitingLP,
  title={Eliciting Latent Predictions from Transformers with the Tuned Lens},
  author={Nora Belrose and Zach Furman and Logan Smith and Danny Halawi and Igor V. Ostrovsky and Lev McKinney and Stella Biderman and Jacob Steinhardt},
  journal={ArXiv},
  year={2023},
  volume={abs/2303.08112},
  url={https://api.semanticscholar.org/CorpusID:257504984}
}

@book{verma2009quantum,
  title={Quantum Physics},
  author={Verma, H.C.},
  isbn={8192571408},
  year={2009},
  publisher={Surya Publications}
}

@misc{3B1B_Grovers_Algorithm,
  author = {Grant Sanderson},
  title = {{But what is quantum computing? (Grover's Algorithm)}},
  howpublished = {YouTube},
  year = {2025},
  note = {Video},
  url = {https://www.youtube.com/watch?v=RQWpF2Gb-gU},
  urldate = {2025-08-07}
}

@article{Erasmus2020WhatII,
  title={What is Interpretability?},
  author={Adrian Erasmus and Tyler D. P. Brunet and Eyal Fisher},
  journal={Philosophy \& Technology},
  year={2020},
  volume={34},
  pages={833 - 862},
  url={https://api.semanticscholar.org/CorpusID:228885094}
}

@article{Zhao2023ExplainabilityFL,
  title={Explainability for Large Language Models: A Survey},
  author={Haiyan Zhao and Hanjie Chen and F. Yang and Ninghao Liu and Huiqi Deng and Hengyi Cai and Shuaiqiang Wang and Dawei Yin and Mengnan Du},
  journal={ACM Transactions on Intelligent Systems and Technology},
  year={2023},
  volume={15},
  pages={1 - 38},
  url={https://api.semanticscholar.org/CorpusID:261530292}
}

@article{Carvalho2019MachineLI,
  title={Machine Learning Interpretability: A Survey on Methods and Metrics},
  author={Diogo Vieira Carvalho and Eduardo Marques Pereira and Jaime S. Cardoso},
  journal={Electronics},
  year={2019},
  url={https://api.semanticscholar.org/CorpusID:199659548}
}

@article{Brown2020LanguageMA,
  title={Language Models are Few-Shot Learners},
  author={Tom B. Brown and Benjamin Mann and Nick Ryder and Melanie Subbiah and Jared Kaplan and Prafulla Dhariwal and Arvind Neelakantan and Pranav Shyam and Girish Sastry and Amanda Askell and Sandhini Agarwal and Ariel Herbert-Voss and Gretchen Krueger and T. J. Henighan and Rewon Child and Aditya Ramesh and Daniel M. Ziegler and Jeff Wu and Clemens Winter and Christopher Hesse and Mark Chen and Eric Sigler and Ma-teusz Litwin and Scott Gray and Benjamin Chess and Jack Clark and Christopher Berner and Sam McCandlish and Alec Radford and Ilya Sutskever and Dario Amodei},
  journal={ArXiv},
  year={2020},
  volume={abs/2005.14165},
  url={https://api.semanticscholar.org/CorpusID:218971783}
}

@article{Dubey2024TheL3,
  title={The Llama 3 Herd of Models},
  author={Abhimanyu Dubey and Abhinav Jauhri and Abhinav Pandey and Abhishek Kadian and Ahmad Al-Dahle and Aiesha Letman and Akhil Mathur and Alan Schelten and Amy Yang and Angela Fan and Anirudh Goyal and Anthony S. Hartshorn and Aobo Yang and Archi Mitra and Archie Sravankumar and others},
  journal={ArXiv},
  year={2024},
  volume={abs/2407.21783},
  url={https://api.semanticscholar.org/CorpusID:271571434}
}

@article{Bricken2023Circuits,
  title={Towards Monosemanticity: Decomposing Language Models With Dictionary Learning},
  author={Trenton Bricken and Adly Templeton and Joshua Batson and Brian Chen and Adam Jermyn and Tom Conerly and Nicholas L Turner and Cem Anil and Carson Denison and Amanda Askell and Robert Lasenby and Yifan Wu and Shauna Kravec and Nicholas Schiefer and Tim Maxwell and Nicholas Joseph and Alex Tamkin and Karina Nguyen and Brayden McLean and Josiah E Burke and Tristan Hume and Shan Carter and Tom Henighan and Chris Olah},
  year={2023},
  jounal = {Anthropic},
  url={https://transformer-circuits.pub/2023/monosemantic-features/index.html}
}

@inproceedings{jawahar-etal-2019-bert,
    title = "What Does {BERT} Learn about the Structure of Language?",
    author = "Jawahar, Ganesh  and
      Sagot, Beno{\^i}t  and
      Seddah, Djam{\'e}",
    editor = "Korhonen, Anna  and
      Traum, David  and
      M{\`a}rquez, Llu{\'i}s",
    booktitle = "Proceedings of the 57th Annual Meeting of the Association for Computational Linguistics",
    month = jul,
    year = "2019",
    address = "Florence, Italy",
    publisher = "Association for Computational Linguistics",
    url = "https://aclanthology.org/P19-1356/",
    doi = "10.18653/v1/P19-1356",
    pages = "3651--3657"
}

@inproceedings{NEURIPS2023_34e1dbe9,
 author = {Conmy, Arthur and Mavor-Parker, Augustine and Lynch, Aengus and Heimersheim, Stefan and Garriga-Alonso, Adri\`{a}},
 booktitle = {Advances in Neural Information Processing Systems},
 editor = {A. Oh and T. Naumann and A. Globerson and K. Saenko and M. Hardt and S. Levine},
 pages = {16318--16352},
 publisher = {Curran Associates, Inc.},
 title = {Towards Automated Circuit Discovery for Mechanistic Interpretability},
 url = {https://proceedings.neurips.cc/paper_files/paper/2023/file/34e1dbe95d34d7ebaf99b9bcaeb5b2be-Paper-Conference.pdf},
 volume = {36},
 year = {2023}
}

@inproceedings{Dalvi2018WhatIO,
  title={What Is One Grain of Sand in the Desert? Analyzing Individual Neurons in Deep NLP Models},
  author={Fahim Dalvi and Nadir Durrani and Hassan Sajjad and Yonatan Belinkov and Anthony Bau and James R. Glass},
  booktitle={AAAI Conference on Artificial Intelligence},
  year={2018},
  url={https://api.semanticscholar.org/CorpusID:56895415}
}

@inproceedings{Rai2024AnIO,
  title={An Investigation of Neuron Activation as a Unified Lens to Explain Chain-of-Thought Eliciting Arithmetic Reasoning of LLMs},
  author={Daking Rai and Ziyu Yao},
  booktitle={Annual Meeting of the Association for Computational Linguistics},
  year={2024},
  url={https://api.semanticscholar.org/CorpusID:270562219}
}

@inproceedings{Devlin2019BERTPO,
  title={BERT: Pre-training of Deep Bidirectional Transformers for Language Understanding},
  author={Jacob Devlin and Ming-Wei Chang and Kenton Lee and Kristina Toutanova},
  booktitle={North American Chapter of the Association for Computational Linguistics},
  year={2019},
  url={https://api.semanticscholar.org/CorpusID:52967399}
}

@article{Ramsauer2020HopfieldNI,
  title={Hopfield Networks is All You Need},
  author={Hubert Ramsauer and Bernhard Schafl and Johannes Lehner and Philipp Seidl and Michael Widrich and Lukas Gruber and Markus Holzleitner and Milena Pavlovi'c and Geir Kjetil Ferkingstad Sandve and Victor Greiff and David P. Kreil and Michael Kopp and G{\"u}nter Klambauer and Johannes Brandstetter and Sepp Hochreiter},
  journal={ArXiv},
  year={2020},
  volume={abs/2008.02217},
  url={https://api.semanticscholar.org/CorpusID:220968978}
}

@article{Li2022QuantumSN,
  title={Quantum Self-Attention Neural Networks for Text Classification},
  author={Guangxi Li and Xuanqiang Zhao and Xin Wang},
  journal={Sci. China Inf. Sci.},
  year={2022},
  volume={67},
  url={https://api.semanticscholar.org/CorpusID:248693324}
}

@inproceedings{jastrzebski2018residual,
title={Residual Connections Encourage Iterative Inference},
author={Stanisław Jastrzebski and Devansh Arpit and Nicolas Ballas and Vikas Verma and Tong Che and Yoshua Bengio},
booktitle={International Conference on Learning Representations},
year={2018},
url={https://openreview.net/forum?id=SJa9iHgAZ},
}

@article{Kim2025OnTE,
  title={On the Effect of Uncertainty on Layer-wise Inference Dynamics},
  author={Sunwoo Kim and Haneul Yoo and Alice Oh},
  journal={ArXiv},
  year={2025},
  volume={abs/2507.06722},
  url={https://api.semanticscholar.org/CorpusID:280138089}
}

@inproceedings{LAD2024StagesOfInference,
title={The Remarkable Robustness of {LLM}s: Stages of Inference?},
author={Vedang Lad and Wes Gurnee and Max Tegmark},
booktitle={ICML 2024 Workshop on Mechanistic Interpretability},
year={2024},
url={https://openreview.net/forum?id=R5unwb9KPc}
}

@inproceedings{Saeidi2016Sentihood,
    title = "{S}enti{H}ood: Targeted Aspect Based Sentiment Analysis Dataset for Urban Neighbourhoods",
    author = "Saeidi, Marzieh  and
      Bouchard, Guillaume  and
      Liakata, Maria  and
      Riedel, Sebastian",
    editor = "Matsumoto, Yuji  and
      Prasad, Rashmi",
    booktitle = "Proceedings of {COLING} 2016, the 26th International Conference on Computational Linguistics: Technical Papers",
    month = dec,
    year = "2016",
    address = "Osaka, Japan",
    publisher = "The COLING 2016 Organizing Committee",
    url = "https://aclanthology.org/C16-1146/",
    pages = "1546--1556"
}

@article{Householder1958Unitary,
author = {Householder, Alston S.},
title = {Unitary Triangularization of a Nonsymmetric Matrix},
year = {1958},
issue_date = {Oct. 1958},
publisher = {Association for Computing Machinery},
address = {New York, NY, USA},
volume = {5},
number = {4},
issn = {0004-5411},
url = {https://doi.org/10.1145/320941.320947},
doi = {10.1145/320941.320947},
journal = {J. ACM},
month = oct,
pages = {339–342},
numpages = {4}
}

@article{Jiao2024AIPhysics,
  author={Licheng Jiao and Xue Song and Chao You and Xu Liu and Lingling Li and Puhua Chen and Xu Tang and Zhixi Feng and Fang Liu and Yuwei Guo and Shuyuan Yang and Yangyang Li and Xiangrong Zhang and Wenping Ma and Shuang Wang and Jing Bai and Biao Hou},
  title={AI meets physics: a comprehensive survey},
  year={2024},
  month={September},
  cdate={1725148800000},
  journal={Artif. Intell. Rev.},
  volume={57},
  number={9},
  pages={256},
  url={https://doi.org/10.1007/s10462-024-10874-4}
}

@article{Dawid2022ModernAO,
  title={Modern applications of machine learning in quantum sciences},
  author={Anna Dawid and Julian Arnold and Borja Requena and Alexander Gresch and Marcin Plodzie'n and Kaelan Donatella and Kim Andrea Nicoli and Paolo Stornati and Rouven Koch and Miriam Buttner and Robert Okuła and Gorka Mu{\~n}oz-Gil and Rodrigo A. Vargas-Hern'andez and Alba Cervera-Lierta and Juan Felipe Carrasquilla and Vedran Dunjko and Marylou Gabri'e and Patrick Huembeli and Evert van Nieuwenburg and Filippo Vicentini and Lei Wang and Sebastian Johann Wetzel and Giuseppe Carleo and Elivska Greplov'a and Roman V. Krems and Florian Marquardt and Michał Tomza and Maciej Lewenstein and Alexandre Dauphin},
  journal={ArXiv},
  year={2022},
  url={https://api.semanticscholar.org/CorpusID:248069577}
}

@article{Kong2025QuantumEnhancedLE,
  title={Quantum-Enhanced LLM Efficient Fine Tuning},
  author={Xiaofei Kong and Lei Li and Meng-Han Dou and Zhaoyun Chen and Yuchun Wu and Guoping Guo},
  journal={ArXiv},
  year={2025},
  volume={abs/2503.12790},
  url={https://api.semanticscholar.org/CorpusID:277066819}
}

@article{Cherrat2024quantumvision,
    doi = {10.22331/q-2024-02-22-1265},
    url = {https://doi.org/10.22331/q-2024-02-22-1265},
    title = {Quantum {V}ision {T}ransformers},
    author = {Cherrat, El Amine and Kerenidis, Iordanis and Mathur, Natansh and Landman, Jonas and Strahm, Martin and Li, Yun Yvonna},
    journal = {{Quantum}},
    volume = {8},
    pages = {1265},
    year = {2024}
}

@article{CHEN2025MixedStateAttention,
title = {Quantum mixed-state self-attention network},
author = {Fu Chen and Qinglin Zhao and Li Feng and Chuangtao Chen and Yangbin Lin and Jianhong Lin},
journal = {Neural Networks},
volume = {185},
pages = {107123},
year = {2025},
issn = {0893-6080},
doi = {https://doi.org/10.1016/j.neunet.2025.107123},
url = {https://www.sciencedirect.com/science/article/pii/S0893608025000024},
}

@article{Zhang2022TransformerQS,
  title={Transformer quantum state: A multipurpose model for quantum many-body problems},
  author={Yuanhua Zhang and Massimiliano Di Ventra},
  journal={Physical Review B},
  year={2022},
  url={https://api.semanticscholar.org/CorpusID:251280410}
}

@article{Lange2024TransformerNN,
  title={Transformer neural networks and quantum simulators: a hybrid approach for simulating strongly correlated systems},
  author={Hannah Lange and Guillaume Bornet and Gabriel Emperauger and Cheng Chen and Thierry Lahaye and Stefan Kienle and Antoine Browaeys and Annabelle Bohrdt},
  journal={Quantum},
  year={2024},
  volume={9},
  pages={1675},
  url={https://api.semanticscholar.org/CorpusID:270220476}
}

@article{Sajjad2020OnTE,
  title={On the effect of dropping layers of pre-trained transformer models},
  author={Hassan Sajjad and Fahim Dalvi and Nadir Durrani and Preslav Nakov},
  journal={Comput. Speech Lang.},
  year={2020},
  volume={77},
  pages={101429},
  url={https://api.semanticscholar.org/CorpusID:251005814}
}

@inproceedings{Ni2019Justifying,
    title = "Justifying Recommendations using Distantly-Labeled Reviews and Fine-Grained Aspects",
    author = "Ni, Jianmo  and
      Li, Jiacheng  and
      McAuley, Julian",
    editor = "Inui, Kentaro  and
      Jiang, Jing  and
      Ng, Vincent  and
      Wan, Xiaojun",
    booktitle = "Proceedings of the 2019 Conference on Empirical Methods in Natural Language Processing and the 9th International Joint Conference on Natural Language Processing (EMNLP-IJCNLP)",
    month = nov,
    year = "2019",
    address = "Hong Kong, China",
    publisher = "Association for Computational Linguistics",
    url = "https://aclanthology.org/D19-1018/",
    doi = "10.18653/v1/D19-1018",
    pages = "188--197"
}

@article{Eldan2023TinyStoriesHS,
  title={TinyStories: How Small Can Language Models Be and Still Speak Coherent English?},
  author={Ronen Eldan and Yuan-Fang Li},
  journal={ArXiv},
  year={2023},
  volume={abs/2305.07759},
  url={https://api.semanticscholar.org/CorpusID:258686446}
}

@article{Kang2018SelfAttentiveSR,
  title={Self-Attentive Sequential Recommendation},
  author={Wang-Cheng Kang and Julian McAuley},
  journal={2018 IEEE International Conference on Data Mining (ICDM)},
  year={2018},
  pages={197-206},
  url={https://api.semanticscholar.org/CorpusID:52127932}
}

@article{Smaldone2025HybridTransformer,
    author = {Smaldone, Anthony M. and Shee, Yu and Kyro, Gregory W. and Farag, Marwa H. and Chandani, Zohim and Kyoseva, Elica and Batista, Victor S.},
    title = {A Hybrid Transformer Architecture with a Quantized Self-Attention Mechanism Applied to Molecular Generation},
    journal = {Journal of Chemical Theory and Computation},
    volume = {21},
    number = {10},
    pages = {5143-5154},
    year = {2025},
    doi = {10.1021/acs.jctc.5c00331},
    note ={PMID: 40333363},
    URL = {https://doi.org/10.1021/acs.jctc.5c00331},
    eprint = {https://doi.org/10.1021/acs.jctc.5c00331}
}

@inproceedings{Black2022GPTNeo,
    title = "{GPT}-{N}eo{X}-20{B}: An Open-Source Autoregressive Language Model",
    author = "Black, Sidney  and
      Biderman, Stella  and
      Hallahan, Eric  and
      Anthony, Quentin  and
      Gao, Leo  and
      Golding, Laurence  and
      He, Horace  and
      Leahy, Connor  and
      McDonell, Kyle  and
      Phang, Jason  and
      Pieler, Michael  and
      Prashanth, Usvsn Sai  and
      Purohit, Shivanshu  and
      Reynolds, Laria  and
      Tow, Jonathan  and
      Wang, Ben  and
      Weinbach, Samuel",
    editor = "Fan, Angela  and
      Ilic, Suzana  and
      Wolf, Thomas  and
      Gall{\'e}, Matthias",
    booktitle = "Proceedings of BigScience Episode {\#}5 -- Workshop on Challenges {\&} Perspectives in Creating Large Language Models",
    month = may,
    year = "2022",
    address = "virtual+Dublin",
    publisher = "Association for Computational Linguistics",
    url = "https://aclanthology.org/2022.bigscience-1.9/",
    doi = "10.18653/v1/2022.bigscience-1.9",
    pages = "95--136"
}

@article{McInnes2018UMAP,
    doi = {10.21105/joss.00861},
    url = {https://doi.org/10.21105/joss.00861},
    year = {2018},
    publisher = {The Open Journal},
    volume = {3},
    number = {29},
    pages = {861},
    author = {McInnes, Leland and Healy, John and Saul, Nathaniel and   Großberger, Lukas},
    title = {UMAP: Uniform Manifold Approximation and Projection},
    journal = {Journal of Open Source Software}
}

@article{Ramos2023Dcor,
  author = {Ramos-Carreño, Carlos and Torrecilla, José L.},
  doi = {10.1016/j.softx.2023.101326},
  journal = {SoftwareX},
  month = {2},
  title = {{dcor: Distance correlation and energy statistics in Python}},
  url = {https://www.sciencedirect.com/science/article/pii/S2352711023000225},
  volume = {22},
  year = {2023},
}

@article{Szekely2007MeasuringAT,
  title={Measuring and testing dependence by correlation of distances},
  author={G'abor J. Sz'ekely and Maria L. Rizzo and Nail K. Bakirov},
  journal={Annals of Statistics},
  year={2007},
  volume={35},
  pages={2769-2794},
  url={https://api.semanticscholar.org/CorpusID:5661488}
}

\end{document}